\documentclass[twocolumn]{autart}    

\usepackage{graphicx}   
\usepackage{amsmath}    
\usepackage{amssymb}    
\usepackage{float}      

\usepackage{verbatim} 

\DeclareMathOperator{\Tr}{Tr} 

\begin{document}

\begin{frontmatter}
\runtitle{Stochastic Port-Hamiltonian Neural Networks}  

\title{Stochastic Port-Hamiltonian Neural Networks: \\ Universal Approximation with Passivity Guarantees} 

\author[UVER]{Luca Di Persio}\ead{luca.dipersio@univr.it},  
\author[BUW]{Matthias Ehrhardt}\ead{ehrhardt@uni-wuppertal.de},         
\author[UTR]{Youness Outaleb}\ead{youness.outaleb@unitn.it} 

\address[UVER]{Department of Computer Science - College of  Mathematics\\ University of Verona, Strada le Grazie 15 - 37134 Verona, Italy}  
\address[BUW]{University of Wuppertal, Chair of Applied and Computational Mathematics,\\
Gau\ss{}strasse 20, 42119 Wuppertal, Germany}             

\address[UTR]{Doctoral School in Mathematics, University of Trento\\
Via Sommarive, 14 - 38123 Povo, Italy}  


\begin{abstract}   
   Stochastic port-Hamiltonian systems represent open dynamical systems with dissipation, inputs, and stochastic forcing in an energy based form. 
   We introduce stochastic port-Hamiltonian neural networks, SPH-NNs, which parameterize the Hamiltonian with a feedforward network and enforce skew symmetry of the interconnection matrix and positive semidefiniteness of the dissipation matrix. 
   For It\^{o} dynamics we establish a weak passivity inequality in expectation under an explicit generator condition, stated for a stopped process on a compact set. 
   We also prove a universal approximation result showing that, on any compact set and finite horizon, SPH-NNs approximate the coefficients of a target stochastic port-Hamiltonian system with $C^2$ accuracy of the Hamiltonian and yield coupled solutions that remain close in mean square up to the exit time. 
   Experiments on noisy mass spring, Duffing, and Van der Pol oscillators show improved long horizon rollouts and reduced energy error relative to a multilayer perceptron baseline.
\end{abstract}

\end{frontmatter}

\section{Introduction}\label{sec:introduction}
Recently, there has been growing interest in using Hamiltonian-based frameworks to model complex physical systems, especially in the field of machine learning for dynamical systems. 
These approaches are appealing because they integrate structural and physical constraints, such as energy conservation and passivity, into data-driven models. 
However, despite their universal approximation capabilities, traditional \textit{neural networks} (NNs) cannot generally enforce fundamental physical laws.
Consequently, they may produce predictions that violate energy conservation or stability constraints, resulting in unphysical behaviors such as uncontrolled energy growth or dissipation in conservative systems.

\subsection{Learning Physics Under Uncertainty}
There is a particular need for data-efficient and physically reliable learning algorithms in domains where safety, interpretability, and extrapolation are essential.  
Examples range from robotics and autonomous systems, where inaccurate energy accounting can lead to catastrophic controller failures, to molecular simulation, where violations of stochastic energetics distort statistical observables.  

In these applications, two goals emerge: learned models should respect intrinsic geometric or energetic properties (e.g., symplecticity, skew-symmetry, positive definiteness, and passivity) that are known a priori. 
However, environmental disturbances, parameter variability, and measurement noise must also be taken into account.  
Deterministic models alone are insufficient when uncertainty alters qualitative behavior.

Port-Hamiltonian theory provides a unifying language that fulfills both criteria.  
It builds upon classical Hamiltonian mechanics by explicitly distinguishing between storage (energy), dissipation (entropy production), and interaction (ports) via Dirac structures.  
When Brownian noise is present, \textit{stochastic port-Hamiltonian systems} (SPHS) are obtained, in which noise enters through additional stochastic ports.  
However, only a few machine learning papers have addressed the full stochastic, energy-aware setting to date.

\subsection{Structure-Preserving Machine Learning}
Many techniques have been proposed to embed physical structure into learning algorithms.  
One such technique is \textit{physics-informed neural networks} (PINNs), which constrain the solutions of partial differential equations (PDEs) by adding the residuals of the governing equations to the loss function \cite{raissi2019physics,karniadakis2021physics}.  
However, vanilla PINNs are not well-suited for long-time integration of Hamiltonian flows because they do not guarantee symplecticity.
\textit{Hamiltonian neural networks} (HNNs) address this issue for conservative, deterministic systems by learning a scalar Hamiltonian $H_\theta$, whose gradients reproduce Hamilton’s equations  \cite{greydanus2019hamiltonian}.

\textit{Port-Hamiltonian neural networks} (PHNNs) are an extension of HNNs that allow for non-conservative interconnections and damping matrices.  
The early work of van der Schaft and Jeltsema \cite{van2014port} inspired data-driven identification methods. 
Regarding model reduction, Breiten and Schulze proposed a linear-quadratic Gaussian (LQG) balanced truncation method for linear port-Hamiltonian descriptor systems that preserves the port-Hamiltonian structure \cite{breiten2025structure}. 
While not a learning algorithm, this method demonstrates how to combine stochastic optimal control ideas with structure preservation, offering a complementary approach to noise-aware modeling.  
However, these contributions typically address deterministic settings and often either overlook noise or add it after training. 
For an overview of recent machine learning techniques for deterministic port-Hamiltonian systems (PHSs), 
see Cherifi \cite{cherifi2020overview}.

\subsection{Related Work on Stochastic Dynamics}\label{subsec:stochastic_related_work}
Over the last five years, the intersection of \textit{stochastic differential equations} (SDEs) and deep learning has given rise to a vibrant research field. 
This field is driven by the need to model noisy physical systems, financial time series, and irregular biomedical records.
Below, we review the relevant literature on structure preservation, emphasizing how each method addresses -- or fails to address -- energy consistency, passivity, and uncertainty.

In their work, Kidger and Lyons \cite{kidger2020neural} cast the drift and diffusion of an It\^{o} SDE as neural networks trained by stochastic adjoint methods, forming the Neural SDE framework. 
Subsequent refinements have improved training stability by expanding state spaces \cite{dupont2019augmented} and by using robust solvers such as Milstein nets \cite{zhang2023milstein}.
However, these black-box approaches, while expressive, disregard the geometric constraints of Hamiltonian flows.

Inspired by deterministic HNNs, Kong \textit{et al.} \cite{kong2020sde} introduced stochastic HNNs (SHNNs), which add isotropic Langevin noise to the canonical equations.  
Di Persio \textit{et al.} \cite{dipersio2025porthamiltonian} developed the Port-Hamiltonian Neural Network (PHNN) framework and extended this concept by providing a systematic way to model both energy dissipation and external interactions (ports).
This overcomes the limitations of gradient-based damping and enables the simulation of interconnected stochastic systems.   
These Hamiltonian-inspired architectures are complemented by \textit{stochastic PINNs} (sPINNs) \cite{o2022stochastic}, which extend the PINN residual loss to stochastic dynamics. 
However, residual-based training alone does not guarantee structure preservation
(e.g., passivity or energy consistency) without additional constraints.

Score-based generative models offer an alternative approach. By training a time-reversed SDE to match the data distribution, one obtains a controllable forward SDE. 
Zhang \textit{et al.} \cite{song2020score} embedded port-Hamiltonian priors into the score network via energy regularizers.
Their sampler enforces passivity only in distribution. 
In contrast, our parameterization (Section~\ref{sec:sphnn}) enforces the structural constraints $J=-J^\top$ and $R\succeq 0$ exactly, and admits localized weak-passivity bounds under an explicit generator inequality (Corollary~\ref{cor:Weak_passivity}).

Current stochastic-dynamic learning methods either ignore the energy structure, are restricted to conservative noise models, or treat passivity as an afterthought.  
The proposed SPH-NN, which integrates 
port-Hamiltonian geometry with flexible neural parameterization, improves long-horizon rollouts and energy-related metrics on the benchmarks considered here relative to a standard MLP baseline.

\subsection{Present Work}
This paper introduces the Stochastic Port-Hamiltonian Neural Network (SPH-NN),
which closes the existing gap. 
Starting from the axiomatic definition of SPHS \cite{cordoni2022stochastic,cordoni2023weak},
we parameterize the Hamiltonian $H_\theta$ with a feedforward network and construct the drift
in port-Hamiltonian form.
In the deterministic case this yields the standard energy balance
\begin{equation*}
    \dot{H}_\theta = u^\top y_\theta - \nabla H_\theta^\top R_\theta \nabla H_\theta.
\end{equation*}
In the It\^{o} stochastic case, the expected energy involves an additional It\^{o} correction
$\tfrac{1}{2}\Tr(\sigma\sigma^\top \nabla^2 H_\theta)$, and weak passivity follows under an
explicit generator inequality (Corollary~\ref{cor:Weak_passivity}).

The dissipative co-matrix $R(x)$ and the interconnection matrix $J(x)$ are assigned to parameterizations that are tractable and maintain positive semi-definiteness and skew-symmetry, respectively. 
The stochastic signature $\sigma(x)$ can be prescribed or learned, enabling fidelity to noisy data.

The result is a learned SDE that retains the port-Hamiltonian structural constraints exactly and admits weak-passivity bounds in expectation (Corollary~\ref{cor:Weak_passivity}), together with universal approximation guarantees (Theorem~\ref{thm:Structured_UAT_SPHS}).  
Figures~\ref{fig:phase_space_mass_spring} and \ref{fig:mse_rollout_mass_spring}
illustrate the benefits of conservative and damped systems affected by nontrivial noise.

\subsection{Contributions}
This paper makes three contributions.
First, we present the first end-to-end differentiable architecture that incorporates the complete stochastic port-Hamiltonian formalism formulated in the It\^{o} sense within a neural network.
Our construction guarantees skew-symmetry and positive semi-definiteness by design, and we study weak passivity in the sense of Definition~\ref{def:strong-weak-passivity} (see Corollary~\ref{cor:Weak_passivity_strict}, Proposition~\ref{prop:Weak_passivity_global}, and Section~\ref{sec:experiments}).

Second, we prove a universal approximation theorem on compact sets: for any compact $\mathcal{K}$ and finite horizon $T$, the SPH-NN parameterization approximates the SPHS coefficients (including $C^2$-accuracy of $H$ on $\mathcal{K}$) and the corresponding coupled solutions are close up to the exit time from $\mathcal{K}$ (Theorem~\ref{thm:Structured_UAT_SPHS}).
We furthermore derive weak-passivity bounds in expectation via an explicit generator inequality:
a localized bound on $\mathcal{K}$ for the stopped process, including an additive term that captures the It\^{o} correction and approximation error (Corollary~\ref{cor:Weak_passivity}), a strict localized version without the additive term under $r\le -\delta$ on $\mathcal{K}$ (Corollary~\ref{cor:Weak_passivity_strict}), and a non-localized weak-passivity result under the additional global condition $\hat{r}\le0$ on $\mathbb{R}^n$ plus integrability assumptions (Proposition~\ref{prop:Weak_passivity_global}).
To the best of our knowledge, this is the first rigorous approximation result coupling stochastic dynamics with a port-Hamiltonian structure.

Third, using canonical benchmarks, we demonstrate that SPH-NNs achieve lower long-term energy drift and more accurate rollouts than \textit{multilayer perceptrons} (MLPs).
In this context, a rollout starts from an initial condition and simulates the learned dynamics forward over many time steps to produce a predicted trajectory, and the \textit{rollout error} quantifies the deviation of this trajectory from the true one over the time horizon.

\subsection{Organization of the Paper}
Section~\ref{sec:theoreticalbackground} reviews the port‑Hamiltonian and stochastic foundations.  
Section~\ref{sec:sphnn} introduces the SPH‑NN architecture and training procedure.  
Section~\ref{sec:experiments} presents numerical experiments and ablation studies.  
Section~\ref{sec:discussion} discusses limitations and perspectives, and Section~\ref{sec:conclusion} concludes.


\section{Theoretical Background} \label{sec:theoreticalbackground}
This section introduces the foundational concepts that underpin our proposed approach. 
First, we review port-Hamiltonian systems.

\subsection{Port-Hamiltonian Systems}
\textit{Port-Hamiltonian systems} (PHSs) are a powerful, structured framework for modeling physical systems involving energy storage, dissipation, and environmental interaction. 
PHSs are particularly well-suited for multiphysics applications where different physical domains are interconnected via energy exchanges.

A standard PHS is described by the following state-space representation:
\begin{equation}\label{eq:PHS}
    \dot{x} = \bigl[J(x) - R(x)\bigr] \nabla H(x) + G(x) u.
\end{equation}
Here, $x \in \mathbb{R}^n$ denotes the state vector, $H(x)\colon\mathbb{R}^n \to \mathbb{R}$ is the Hamiltonian function representing the total stored energy in the system, $J(x) = -J(x)^\top$ is a skew-symmetric matrix encoding the structure of the internal energy-conserving interconnections, 
and $R(x)$ is a positive semidefinite matrix modeling dissipative effects. 
The influence of external inputs $u\in\mathbb{R}^m$ is captured by the control input matrix $G(x)$.

A key property of PHSs 
is \textit{passivity}, which ensures that the system cannot generate energy internally. 
Instead, the system can only store, dissipate, or exchange energy with its environment. 
Passivity plays a central role in analyzing system stability and designing robust, physically consistent controllers. 
Port-Hamiltonian models naturally extend to networked and distributed systems, making them valuable tools for theoretical analysis and practical engineering applications.

\subsection{Stochastic Extension}
Uncertainties, measurement noise, and unmodeled external disturbances can significantly impact the dynamics of many real-world systems. 
To account for these effects, the port-Hamiltonian framework can be extended to include stochastic perturbations, resulting in a \textit{stochastic port-Hamiltonian system} (SPHS) \cite{cordoni2022stochastic}.

A typical SPHS is modeled by introducing a stochastic term driven by a Brownian motion $W_t$ into the deterministic formulation \eqref{eq:PHS}:
\begin{equation}\label{eq:SPHS}
\begin{split}
dX_t
&= \bigl[J(X_t)-R(X_t)\bigr]\nabla H(X_t)\,dt \\
&\qquad + g(X_t)u_t\,dt + \sigma(X_t)\,dW_t,
\end{split}
\end{equation}
where $X_t\in\mathbb{R}^n$ denotes the system state at time $t$, 
$J(X_t)$ is a skew-symmetric interconnection matrix, 
$R(X_t)$ is a positive semidefinite dissipation matrix, 
$H(X_t)$ is the Hamiltonian representing stored energy, 
$y(X_t)=g(X_t)^\top \nabla H(X_t)$ is the output,
and $\sigma(X_t)$ is the diffusion matrix that characterizes how noise enters the system.

This stochastic formulation preserves the key structural components of PHSs, 
namely the separation of conservative, dissipative, and external effects, while accommodating random fluctuations. 
This extension is essential for accurately modeling physical systems subject to uncertainty, and ensures that important properties, such as passivity and energy consistency, are retained in expectation or in a probabilistic sense.

\subsection{Energy Principles and Passivity}
The principle of energy conservation and balance is a central concept in PHS theory. 
The Hamiltonian function $H$ represents the system's total stored energy, while the dissipation matrix $R(x)$ models energy loss due to resistive effects.
In deterministic PHSs, the system's structure ensures that it cannot generate energy internally, thereby satisfying the principle of \textit{passivity}.

In the It\^{o} interpretation, applying It\^{o}'s formula to $H(X_t)$ produces an additional second-order term $\tfrac{1}{2} \Tr(\sigma \sigma^\top\nabla^2 H)$ in the energy balance, so noise may inject or remove energy in expectation even when $R \equiv 0$.
Consequently, weak passivity is not automatic from $J = -J^\top$ and $R \succeq 0$ alone, and it requires an explicit generator inequality that balances dissipation and the It\^{o} correction (see Corollary~\ref{cor:Weak_passivity} and Proposition~\ref{prop:Weak_passivity_global}).

Several formal notions of stochastic passivity have been proposed in the literature, see for example, \cite{cordoni2022stochastic,cordoni2023weak}. 
For the sake of completeness, we recall the standard definitions of \textit{strong} and \textit{weak} passivity.
\begin{defn}[Strong and Weak Passivity]\label{def:strong-weak-passivity}
Let $H\colon\mathcal{X} \to \mathbb{R}$ denote the Hamiltonian of a SPHS. 
The system interacts with its environment through port variables $(u,y)$, where the input is $u \colon [0,\infty) \to \mathbb{R}^m$ and the output is $y \colon [0,\infty) \to \mathbb{R}^m$. Then:
\begin{enumerate}
    \item The system is \textit{strongly passive} \cite{cordoni2022stochastic} if, for all sample paths and all $t \geq 0$,
    \begin{equation}\label{eq:inequality3}
        H(X_t) \leq H(X_0) + \int_0^t u(s)^\top y(s)\, ds.
    \end{equation}
    
    \item The system is \textit{weakly passive}  \cite{cordoni2022stochastic} if the inequality \eqref{eq:inequality3} holds in expectation:
    \begin{align}
        \mathbb{E}\bigl[H(X_t)\bigr] \leq \mathbb{E}\bigl[H(X_0)\bigr] + \mathbb{E}\biggl[ \int_0^t u(s)^\top y(s)\, ds \biggr].
    \end{align}
\end{enumerate}
\end{defn}

These definitions ensure that, on average, the energy within the system does not exceed the initial energy plus the cumulative power supplied through the input ports. 
In a stochastic setting, these passivity constraints are crucial for ensuring stability and limiting the system’s response under random perturbations.

\subsection{Learning a Hamiltonian from Data}
A key concept in structure-preserving machine learning is learning the underlying Hamiltonian function directly from the data. 
This captures the physics that govern the system. 
In the deterministic setting, the dynamics of many physical systems can be described by \textit{Hamilton's equations}:
\begin{equation}
    \dot{\mathbf{q}} = \frac{\partial \mathcal{H}_\theta}{\partial \mathbf{p}}, \qquad
    \dot{\mathbf{p}} = -\frac{\partial \mathcal{H}_\theta}{\partial \mathbf{q}},
\end{equation}
where $\mathcal{H}_\theta$ is a neural network parameterization of the Hamiltonian function, and $(\mathbf{q}, \mathbf{p})$ denote the generalized coordinates and momenta.

Given observed trajectory data $\{ (\mathbf{q}, \mathbf{p}, \dot{\mathbf{q}}, \dot{\mathbf{p}}) \}$, one can train the neural network $\mathcal{H}_\theta$ 
by minimizing the discrepancy between the predicted and true time derivatives using the following loss function \cite{greydanus2019hamiltonian}:
\begin{equation}\label{eq:objectivefunction}
    L_{\rm HNN} = \Bigl\| \dot{\mathbf{q}} - \frac{\partial \mathcal{H}_\theta}{\partial \mathbf{p}} \Bigr\|^2 
    + \Bigl\| \dot{\mathbf{p}} + \frac{\partial \mathcal{H}_\theta}{\partial \mathbf{q}} \Bigr\|^2.
\end{equation}
This objective function \eqref{eq:objectivefunction} motivates the learned model to satisfy Hamilton’s equations. 
It embeds the system's symplectic structure, which improves long-term predictive stability and energy conservation compared to standard regression techniques. 
This methodology forms the basis of the \textit{Hamiltonian Neural Network} (HNN) framework, introduced in \cite{greydanus2019hamiltonian}.

To extend this idea to stochastic systems, we consider the drift component of a SPHS \eqref{eq:SPHS}, given by 
\begin{equation}\label{eq:SPHSpart}
    \bigl[J(X_t) - R(X_t)\bigr] \nabla H_\theta(X_t),
\end{equation}
while a diffusion term $\sigma(X_t)\,dW_t$ accounts for stochastic perturbations (noise), and is usually modeled as Brownian motion $W_t$.
Given sampled trajectories of the system state $X_t$, we define a training objective that captures the drift structure:
\begin{equation}\label{eq:sph-nn-loss}
  \arg\min_\theta \Bigl\| \dot{X}_t - \bigl[J(X_t) - R(X_t)\bigr] \nabla H_\theta(X_t) \Bigr\|^2.
\end{equation}

In the stochastic setting, $X_t$ is an It\^{o} process that does not admit a classical derivative $\tfrac{dX_t}{dt}$. 
We can interpret $\dot{X}$ in two different ways:

\begin{enumerate}
\item Drift-based interpretation: $\dot{X}$ could be read as the empirical velocity computed from finite-difference estimates of trajectory increments,
\begin{equation}
  \dot{X}_t \approx \frac{x_{t+\Delta t} - x_t}{\Delta t},
\end{equation}
where $x_t$ are data points. 
This approach is used in \cite{greydanus2019hamiltonian} when analytic derivatives are unavailable. Instead, the authors used finite-difference approximations of $\dot{\mathbf{q}},\dot{\mathbf{p}}$, where $\text{q}$ is the position, $\text{p}$ is the momentum and $X_t=(\text{q},\text{p})$.
\item Conditional expectation interpretation: we match the expected increment per unit time,
\begin{equation}
  \dot{X}_t \approx \mathbb{E} \biggl[ \frac{X_{t+\Delta}-X_t}{\Delta} \Big| X_t\biggr].
\end{equation}
This approach is used in \cite{boninsegna2018sparse} where SDEs are learned by regressing Kramers-Moyal conditional moments.
\end{enumerate}

The formulation \eqref{eq:sph-nn-loss} is the cornerstone of our stochastic port-Hamiltonian neural network (SPH-NN) framework. It allows for the learning of energy-consistent dynamics under uncertainty by expressing the discrepancy in the drift term as an explicit loss function.

\subsection{The SDE Stability Lemma}
Several of our theoretical results depend on the continuity of SDE solutions with respect to small perturbations in the drift and diffusion terms. 
In stochastic calculus, it is well-known that the strong solution of an SDE depends continuously on its coefficients (and initial data) under standard conditions (Lipschitz continuity). 
If two SDEs have drift and diffusion coefficients that are close to each other (e.g., in sup norm) and share the same Brownian motion and initial state, then their solution processes will also be close to each other.

\begin{lem}\label{lem:PH_structured_approx}
Let $\mathcal{K} \subset \mathbb{R}^n$ be compact. 
Let $J\colon\mathcal{K} \to \mathbb{R}^{n \times n}$ and $\sigma\colon\mathcal{K} \to \mathbb{R}^{n \times d}$ be continuous and assume $J(x) = - J(x)^\top$ for all $x \in \mathcal{K}$. 
Let $H \in C^2(\mathbb{R}^n)$. 
Assume there exist $r \in \mathbb{N}$ and a continuous map $D\colon\mathcal{K} \to \mathbb{R}^{r \times n}$ and define $R(x) = D(x)^\top D(x)$ for $x \in \mathcal{K}$. 
Assume the activation $G$ is $l$-finite in the sense of \cite{hornik1990universal} for some $l \in \mathbb{N}$ with $l > 2$.
Throughout, $\|\cdot\|$ denotes the Euclidean norm for vectors and the induced operator norm for matrices, and $\|\cdot\|_F$ denotes the Frobenius norm.

Then, for any $\varepsilon > 0$ there exist feedforward neural networks with activation $G$
\begin{align*}
  \hat{A}\colon\mathbb{R}^n \to \mathbb{R}^{n \times n},\quad
  \hat{D}\colon\mathbb{R}^n \to \mathbb{R}^{r \times n},\\
   \hat{\sigma}\colon\mathbb{R}^n \to \mathbb{R}^{n \times d},\quad
   \hat{H}\colon\mathbb{R}^n \to \mathbb{R},
\end{align*}
such that, defining $\hat{J}(x) = \hat{A}(x) - \hat{A}(x)^\top$ and $\hat{R}(x) = \hat{D}(x)^\top \hat{D}(x)$, we have:
\begin{enumerate}
\item $\hat{J}(x) = -\hat{J}(x)^\top$ and $\hat{R}(x)\succeq 0$ for all $x \in \mathbb{R}^n$.
\item
\begin{equation*}
\begin{split}
& \sup_{x \in \mathcal{K}}\Big(
\|\hat{J}(x)-J(x)\|+\|\hat{R}(x)-R(x)\| \\
& \qquad + \|\hat{\sigma}(x)-\sigma(x)\|
\Big) \\
& \qquad + \|\hat{H}-H\|_{C^2(\mathcal{K})} \leq \varepsilon,
\end{split}
\end{equation*}
where
\begin{equation*}
\begin{split}
\|\hat{H}-H\|_{C^2(\mathcal{K})}
&= \sup_{x \in \mathcal{K}}|\hat{H}(x)-H(x)| \\
&\quad + \sup_{x \in \mathcal{K}}\|\nabla\hat{H}(x)-\nabla H(x)\| \\
&\quad + \sup_{x \in \mathcal{K}}\|\nabla^2\hat{H}(x)-\nabla^2 H(x)\|.
\end{split}
\end{equation*}
\end{enumerate}
\end{lem}

\begin{pf}
Since $\mathcal{K}$ is compact it is closed, and each scalar entry of $J$, $D$, $\sigma$ is bounded on $\mathcal{K}$. 
By the Tietze extension theorem applied to each scalar entry (componentwise), extend $J$, $D$, $\sigma$ from $\mathcal{K}$ to continuous maps on $\mathbb{R}^n$ (still denoted $J$, $D$, 
$\sigma$). 
Use the norm bounds
\begin{equation*}
   \|M\| \le \|M\|_F \leq \sqrt{pq} \max_{i,j}|M_{ij}|\quad \text{for } M \in \mathbb{R}^{p \times q}.
\end{equation*}
Let $M_D = \sup_{x \in \mathcal{K}}\|D(x)\| < \infty$ and set
\begin{equation*}
\eta_J = \frac{\varepsilon}{4},\; \eta_\sigma = \frac{\varepsilon}{4},\; \eta_H = \frac{\varepsilon}{4},\; \eta_D = \min\Big\{1,\frac{\varepsilon}{4(2M_D+1)}\Big\}.
\end{equation*}
Choose
\begin{equation*}
    \alpha_A = \frac{\eta_J}{2n}, \; \alpha_D=\frac{\eta_D}{\sqrt{rn}}, \; \alpha_\sigma=\frac{\eta_\sigma}{\sqrt{nd}}, \; \alpha_H=\frac{\eta_H}{1 + \sqrt{n} + n}.
\end{equation*}
By \cite[Corollary~3.4]{hornik1990universal} with $m=0$, choose networks $\hat{A}$, $\hat{D}$, $\hat{\sigma}$ such that on $\mathcal{K}$,
\begin{align*}
 \max_{i,j}\Big|\hat{A}_{ij}(x)-\frac{1}{2}J_{ij}(x)\Big|
&\le \alpha_A, \\
\max_{i,j}\big|\hat{D}_{ij}(x)-D_{ij}(x)\big|
&\le \alpha_D, \\
 \max_{i,j}\big|\hat{\sigma}_{ij}(x)-\sigma_{ij}(x)\big|
&\le \alpha_\sigma.
\end{align*}
Then for $x \in \mathcal{K}$, we have
\begin{equation*}
  \Big\|\hat{A}(x) - \frac{1}{2} J(x) \Big\| 
  \le \Big\|\hat{A}(x) - \frac{1}{2}J(x)\Big\|_F \le n \alpha_A = \frac{\eta_J}{2},
\end{equation*}
\begin{equation*}
   \|\hat{D}(x) - D(x)\| \le \|\hat{D}(x) - D(x)\|_F \le \sqrt{rn} \alpha_D = \eta_D,
\end{equation*}
\begin{equation*}
   \|\hat{\sigma}(x) - \sigma(x)\| \le \|\hat{\sigma}(x) - \sigma(x)\|_F \le \sqrt{nd} \alpha_\sigma = \eta_\sigma.
\end{equation*}
By \cite[Corollary~3.4]{hornik1990universal} with $m = 2$, choose $\hat{H}$ such that for all multi-indices $\beta$ with $|\beta| \leq 2$,
\begin{equation*}
   \sup_{x \in \mathcal{K}}|\partial^\beta\hat{H}(x) - \partial^\beta H(x)| \leq \alpha_H.
\end{equation*}
Hence
\begin{align*}
\sup_{x \in \mathcal{K}}|\hat{H}(x)-H(x)|
&\leq \alpha_H, \\
 \sup_{x \in \mathcal{K}}\|\nabla\hat{H}(x)-\nabla H(x)\|
&\leq \sqrt{n}\alpha_H, \\
 \sup_{x \in \mathcal{K}}\|\nabla^2\hat{H}(x)-\nabla^2 H(x)\|
&\leq n\alpha_H,
\end{align*}
so $\|\hat{H} - H\|_{C^2(\mathcal{K})} \leq (1 + \sqrt{n} + n)\alpha_H = \eta_H$.

Next, define $\hat{J} = \hat{A} - \hat{A}^\top$ and $\hat{R} = \hat{D}^\top\hat{D}$. 
Then, $\hat{J}$ is skew-symmetric and $\hat{R} \succeq 0$ on $\mathbb{R}^n$.
For $x \in \mathcal{K}$, since $J(x) =-J(x)^\top$,
\begin{equation*}
\begin{split}
\|\hat{J}(x)-J(x)\|
&= \Big\|\Big(\hat{A}(x)-\frac{1}{2}J(x)\Big)
- \Big(\hat{A}(x)-\frac{1}{2}J(x)\Big)^\top\Big\| \\
&\le 2\Big\|\hat{A}(x)-\frac{1}{2}J(x)\Big\| \le\eta_J.
\end{split}
\end{equation*}

Also, for $x \in \mathcal{K}$,
\begin{equation*}
\begin{split}
\|\hat{R}(x)-R(x)\|
&= \|\hat{D}(x)^\top\hat{D}(x)-D(x)^\top D(x)\| \\
&\leq (\|\hat{D}(x)\|+\|D(x)\|)\,
\|\hat{D}(x)-D(x)\|.
\end{split}
\end{equation*}

Since $\|\hat{D}(x)\| \leq \|D(x)\| + \eta_D \leq M_D + 1$, we get
\begin{equation*}
   \|\hat{R}(x) - R(x)\| \leq (2M_D + 1)\eta_D \leq \frac{\varepsilon}{4}.
\end{equation*}
Therefore,
\begin{equation*}
\begin{split}
& \sup_{x \in \mathcal{K}}\Big(
\|\hat{J}(x)-J(x)\|+\|\hat{R}(x)-R(x)\| \\
& \qquad + \|\hat{\sigma}(x)-\sigma(x)\|
\Big) + \|\hat{H}-H\|_{C^2(\mathcal{K})} \\
& \le \eta_J+\frac{\varepsilon}{4}+\eta_\sigma+\eta_H=\varepsilon.
\end{split}
\end{equation*}

\end{pf}

\begin{prop}\label{prop:Stopped_SDE_stability}
Fix $T > 0$ and a compact set $\mathcal{K} \subset \mathbb{R}^n$. 
Let $b,\hat{b}\colon[0,T] \times \mathbb{R}^n \to\mathbb{R}^n$ be measurable and continuous in $x$, and let $\sigma,\hat{\sigma}\colon\mathbb{R}^n \to \mathbb{R}^{n \times d}$ be continuous. 
Let $X$, $\widehat{X}$ be strong solutions on $[0,T]$ of
\begin{equation*}
   dX_t = b(t,X_t) \, dt + \sigma(X_t) \,dW_t, \quad X_0=x_0 \in \mathcal{K},
\end{equation*}
\begin{equation*}
   d \widehat{X}_t = \hat{b}(t,\widehat{X}_t) \, dt + \hat{\sigma}(\widehat{X}_t) \, dW_t, \quad \widehat{X}_0=x_0,
\end{equation*}
driven by the same Brownian motion $W$. Assume:
\begin{enumerate}
\item There exists $L < \infty$ such that for all $x , y \in \mathcal{K}$ and all $t \in [0,T]$,
\begin{equation*}
     |b(t,x) - b(t,y)| + \|\sigma(x) - \sigma(y)\|_F \leq L|x-y|.
\end{equation*}
\item
\begin{align*}
\alpha &= \sup_{t \in [0,T]}\sup_{x\in\mathcal{K}}|b(t,x)-\hat{b}(t,x)| < \infty, \\
\beta &= \sup_{x \in \mathcal{K}}\|\sigma(x)-\hat{\sigma}(x)\|_F < \infty.
\end{align*}

\end{enumerate}
Define the joint exit time
\begin{equation*}
   \tau_{\mathcal{K}} = \inf\{t \ge0:X_t \notin\mathcal{K}\text{ or }\widehat{X}_t \notin\mathcal{K}\}.
\end{equation*}
With $\Delta_t = X_{t\wedge\tau_{\mathcal{K}}} - \widehat{X}_{t \wedge \tau_{\mathcal{K}}}$ and $F(t) = \mathbb{E} \big[\sup_{0 \leq r \leq t}|\Delta_r|^2\big]$, for all $t \in [0,T]$,
\begin{equation*}
   F(t) \le (4T + 16)L^2 \int_0^tF(s) \, ds + (4T^2+16T) (\alpha+\beta)^2,
\end{equation*}
hence
\begin{equation*}
   F(T) \le (4T^2 + 16T) \,\exp\bigl((4T + 16)L^2T\bigr)(\alpha +\beta)^2.
\end{equation*}
In particular, for every $\varepsilon > 0$,
\begin{equation*}
\begin{split}
&\mathbb{P}\Big(\sup_{0 \leq t \leq T \wedge \tau_{\mathcal{K}}}
|X_t-\widehat{X}_t| > \varepsilon
\Big)\\ \le
& \frac{(4T^2+16T)\,\exp\bigl((4T+16)L^2T\bigr)}{\varepsilon^2} \times (\alpha+\beta)^2.
\end{split}
\end{equation*}
\end{prop}

\begin{pf}
Since $X,\widehat{X}$ are continuous adapted processes and $\mathcal{K}\times\mathcal{K}$ is closed, $\tau_{\mathcal{K}}$ is a stopping time and $X_{t\wedge\tau_{\mathcal{K}}},\widehat{X}_{t\wedge\tau_{\mathcal{K}}}\in\mathcal{K}$ for all $t \in [0,T]$. 
In particular, the stopped integrands below are predictable and bounded, hence square-integrable, so the It\^o integrals are well-defined square-integrable martingales.

For $t \in [0,T]$,
\begin{equation*}
\begin{split}
\Delta_t
&= \int_0^{t \wedge \tau_{\mathcal{K}}}\bigl(b(s,X_s) - \hat{b}(s,\widehat{X}_s)\bigr)\,ds \\
&\quad + \int_0^{t \wedge \tau_{\mathcal{K}}}\bigl(\sigma(X_s) - \hat{\sigma}(\widehat{X}_s)\bigr)\,dW_s .
\end{split}
\end{equation*}

Set
\begin{align*}
A_t
&= \int_0^{t \wedge \tau_{\mathcal{K}}}\bigl(b(s,X_s) - \hat{b}(s,\widehat{X}_s)\bigr)\,ds, \\
M_t
&= \int_0^{t \wedge \tau_{\mathcal{K}}}\bigl(\sigma(X_s) - \hat{\sigma}(\widehat{X}_s)\bigr)\,dW_s ,
\end{align*}
so that $\Delta_t = A_t + M_t$ and
\begin{equation*}
\sup_{0 \leq r \leq t}|\Delta_r|^2
\leq 2\sup_{0 \leq r \leq t}|A_r|^2 + 2\sup_{0 \le r \le t}|M_r|^2.
\end{equation*}

For $s < \tau_{\mathcal{K}}$ we have $X_s,\widehat{X}_s \in \mathcal{K}$, so
\begin{equation*}
\begin{split}
|b(s,X_s)-\hat{b}(s,\widehat{X}_s)|
&\leq |b(s,X_s)-b(s,\widehat{X}_s)| \\
&\quad + |b(s,\widehat{X}_s)-\hat{b}(s,\widehat{X}_s)| \\
&\leq L|\Delta_s|+\alpha,
\end{split}
\end{equation*}
and
\begin{equation*}
\begin{split}
\|\sigma(X_s)-\hat{\sigma}(\widehat{X}_s)\|_F
&\leq \|\sigma(X_s)-\sigma(\widehat{X}_s)\|_F \\
&\quad + \|\sigma(\widehat{X}_s)-\hat{\sigma}(\widehat{X}_s)\|_F \\
&\leq L|\Delta_s|+\beta.
\end{split}
\end{equation*}

For the drift term, for every $r \in [0,t]$,
\begin{equation*}
\begin{split}
|A_r|
&\le \int_0^{r \wedge \tau_{\mathcal{K}}}
\bigl|b(s,X_s)-\hat{b}(s,\widehat{X}_s)\bigr|\,ds \\
&\le \int_0^{t \wedge \tau_{\mathcal{K}}}
\bigl(L|\Delta_s|+\alpha\bigr)\,ds ,
\end{split}
\end{equation*}
so by Cauchy--Schwarz,
\begin{equation*}
\begin{split}
\sup_{0 \le r \le t}|A_r|^2
&\le \Big(\int_0^{t \wedge \tau_{\mathcal{K}}}
\bigl(L|\Delta_s|+\alpha\bigr)\,ds\Big)^2 \\
&\le t \int_0^{t \wedge \tau_{\mathcal{K}}}
\bigl(L|\Delta_s|+\alpha\bigr)^2\,ds .
\end{split}
\end{equation*}

Taking expectations and using $(a+b)^2 \leq 2a^2+2b^2$ and $|\Delta_s|^2 \leq \sup_{0 \leq u \leq s}|\Delta_u|^2$ gives
\begin{equation*}
\begin{split}
\mathbb{E}\Big[\sup_{0 \leq r \leq t}|A_r|^2\Big]
&\leq t\int_0^t \mathbb{E}\big[(L|\Delta_s|+\alpha)^2\big]\,ds \\
&\leq t\int_0^t \big(2L^2F(s) + 2\alpha^2\big)\,ds \\
&\leq 2TL^2\int_0^tF(s)\,ds + 2T^2\alpha^2.
\end{split}
\end{equation*}

For the martingale term, Doob's $L^2$ maximal inequality and It\^o isometry yield
\begin{equation*}
\begin{split}
\mathbb{E}\Big[\sup_{0 \leq r \leq t}|M_r|^2\Big]
&\leq 4\,\mathbb{E}|M_t|^2 \\
&= 4\,\mathbb{E}\int_0^{t \wedge \tau_{\mathcal{K}}}\|\sigma(X_s)-\hat{\sigma}(\widehat{X}_s)\|_F^2\,ds \\
&\leq 4\int_0^t \mathbb{E}\big[(L|\Delta_s|+\beta)^2\big]\,ds \\
&\leq 4\int_0^t \big(2L^2F(s) + 2\beta^2\big)\,ds \\
&\leq 8L^2\int_0^tF(s)\,ds + 8T\beta^2.
\end{split}
\end{equation*}

Combining the two estimates, we get
\begin{equation*}
\begin{split}
F(t)
&=\mathbb{E} \Big[\sup_{0 \leq r \leq t}|\Delta_r|^2\Big] \\
&\leq 2\mathbb{E}\Big[\sup_{0 \leq r \leq t}|A_r|^2\Big]
+2\mathbb{E}\Big[\sup_{0 \leq r \leq t}|M_r|^2\Big] \\
&\leq (4T + 16)L^2\int_0^tF(s) \, ds + 4T^2\alpha^2 + 16T\beta^2 \\
&\leq (4T + 16)L^2\int_0^tF(s) \, ds + (4T^2+16T)(\alpha+\beta)^2.
\end{split}
\end{equation*}

Gronwall's inequality yields the bound on $F(T)$. Finally, Markov's inequality gives
\begin{multline*}
\mathbb{P}\Big(
\sup_{0 \leq t \leq T \wedge \tau_{\mathcal{K}}}
|X_t-\widehat{X}_t| > \varepsilon
\Big)\\
= \mathbb{P}\Big(\sup_{0 \leq r \leq T}|\Delta_r|^2 > \varepsilon^2\Big)
\leq \frac{F(T)}{\varepsilon^2},
\end{multline*}
which yields the stated probability estimate.
\end{pf}

\section{Stochastic Port-Hamiltonian Neural Networks}\label{sec:sphnn}
We now formalize our proposed model: a class of neural networks–based stochastic port-Hamiltonian systems (SPHS), referred to as a stochastic port-Hamiltonian neural networks (SPH-NNs).
The key idea is to use a neural network to parameterize the system's Hamiltonian. 
This allows us to learn energy-consistent dynamics directly from data while enforcing key port-Hamiltonian structural constraints. Weak-passivity bounds in expectation follow under explicit generator inequalities (Corollary~\ref{cor:Weak_passivity}).

Given a time series $\{X_t\}_{t=0}^T$, sampled at intervals with a step size $\Delta$, we fit the drift
\begin{equation}
    \dot{X}_t \approx \bigl[ J(X_t) - R(X_t) \bigr] \,\nabla H_\theta(X_t)
\end{equation}
by minimizing the squared error between the model prediction and a data-driven target for $\dot{X}$. 
We consider two targets:

\begin{enumerate}
\item \textbf{Increment-Based (IB):}
\begin{equation}\label{eq:fd-target}
    \widehat{\dot{X}}^{\text{IB}}_t = \frac{X_{t + \Delta}-X_t}{\Delta},
\end{equation}

\item \textbf{Conditional Expectation (CE):}
\begin{equation}\label{eq:ce-target}
   \widehat{\dot{X}}^{\text{CE}}(x) \approx \mathbb{E} \; \biggl[ \frac{X_{t + \Delta}-X_t}{\Delta} \, \Big| \, X_t = x\biggr],
\end{equation}
estimated by local regression of increments against state.
\end{enumerate}

The corresponding SPH-NN objectives are
\begin{equation}\label{eq:loss-fd}
   \mathcal{L}_{\text{IB}}(\theta) 
   = \frac{1}{T}\sum_{t = 0}^{T-1}\Bigl\| \widehat{ \dot{X}}^{\text{IB}}_t 
      - \bigl[J(X_t)-R(X_t)\bigr] \nabla H_\theta(X_t)\Bigr\|^2,
\end{equation}
\begin{multline}\label{eq:loss-km}
\mathcal{L}_{\text{CE}}(\theta) =\\
\frac{1}{T} \sum_{t=0}^{T-1} \Bigl\|\widehat{\dot{X}}^{\text{CE}}(X_t) - \bigl[J(X_t) - R(X_t) \bigr] \nabla H_\theta(X_t)\Bigr\|^2.
\end{multline}

We also parameterize the diffusion $B_\theta(x)$ and train it using the negative log-likelihood under an Euler-Maruyama discretization. 
With an Euler-Maruyama discretization, we obtain the discrete-time update
\begin{equation*}
X_{k+1} \approx X_k + f(X_k)\,\Delta t + B(X_k)\sqrt{\Delta t}\,\xi_k,
\quad \xi_k \sim \mathcal{N}(0,I),
\end{equation*}
such that
\begin{equation*}
\Delta X_k := X_{k+1} - X_k
\sim \mathcal{N}\big(f(X_k)\Delta t,B (X_k)B (X_k)^{\top} \Delta t\big),
\end{equation*}
with mean $\mu_{\theta}(X_k) = f(X_k)\Delta t$ and covariance
$\Sigma_{\theta}(X_k) = B_{\theta}(X_k)B_{\theta}(X_k)^{\top} \Delta t$.

Given data $(X_k, X_{k+1})$, the network is trained by \textit{negative log-likelihood} (NLL):
\begin{equation}\label{eq:NLL}
 \mathcal{L}_{\text{NLL}}(\theta)
= - \sum_k \log
\mathcal{N}\bigl(\Delta X_k \mid \mu_{\theta}(X_k),\Sigma_{\theta}(X_k)\bigr).
\end{equation}

The NLL objective \eqref{eq:NLL} is used to fit the full conditional law of the increments $\Delta X_k$ by learning the mean $\mu_{\theta}(x) = f_{\theta}(x)\Delta t$ and the covariance $\Sigma_{\theta}(x) = B_{\theta}(x)B_{\theta}(x)^{\top}\Delta t$ that best explain the data. 
Since the Gaussian model arises from the Euler-Maruyama scheme, it should be interpreted as an approximation of the continuous-time transition density whose fidelity improves as $\Delta t \to 0$ (under standard regularity assumptions).

The SPH-NN architecture remains the same in all experiments, except for the varying training objective. 
Table~\ref{tab:sph-nn-losses} summarizes the three resulting variants and their associated losses.

\begin{table}[htb]
\centering
\begin{tabular}{ll}
\hline
Variant & Objective \\
\hline
SPH-NN-IB  & \!\!\!\!Increment-based loss $\mathcal{L}_{\text{IB}}$ \\
SPH-NN-CE  & \!\!\!\!Conditional-expectation loss $\mathcal{L}_{\text{CE}}$ \\
SPH-NN-NLL & \!\!\!\!Euler-Maruyama 
neg.\ log-likelihood $\mathcal{L}_{\text{NLL}}$ \\
\hline
\end{tabular}
\caption{SPH-NN variants trained with three objectives.}
\label{tab:sph-nn-losses}
\end{table}

For all variants, the gradient of the loss function $\nabla H_\theta$, or the derivative of the loss function with respect to the parameters, is computed by automatic differentiation. The parameters are then optimized with Adam.
The system state, $x = [q, p] \in \mathbb{R}^n$, serves as the input to the network. 
Here, $q$ and $p$ represent the generalized positions and momenta, respectively. 
The network architecture consists of two hidden layers, each with 64 neurons and Tanh activation functions. These are followed by a linear output layer that produces a scalar-valued Hamiltonian. 
The network is trained using the Adam optimizer \cite{kingma2014adam}, with gradients computed via automatic differentiation. 
In the benchmarks below, the structure-preserving formulation improves long-horizon behavior relative to an unconstrained MLP baseline.

In the stochastic setting with diffusion, individual trajectories neither conserve energy nor satisfy strict passivity,
because random perturbations can inject or remove energy. 
The evolution of energy is determined by the balance between dissipation and noise-induced effects.
The SPH-NN satisfies key properties similar to those of port-Hamiltonian systems. 
When dissipation is absent ($R(x)=0$), the It\^{o} correction can inject energy in expectation. In Section~\ref{sec:experiments} we report the empirical mean energy over the rollout horizon $[0,T]$, where it remains bounded.


\section{Experimental Results}\label{sec:experiments}
We evaluate the proposed structure-preserving learning approach using three stochastic port-Hamiltonian benchmarks: 
a linear mass-spring oscillator, the Duffing oscillator, and a stochastic Van der Pol (VDP) oscillator.
In Table~\ref{tab:oscillator_results}, we compare a multilayer perceptron baseline (Baseline) against the three SPH-NN variants: SPH-NN-IB/CE/NLL. 
Unless otherwise stated, we fix the analytical interconnection matrix $J$ (and $g$, when present).
For mass-spring and Duffing we learn $H_\theta$ and compare it against the analytic Hamiltonian.
For Van der Pol, which is not canonically Hamiltonian in 2D and for which the choice of storage is non-unique, we fix $H(X_t)=\frac{1}{2}X_t^\top I X_t$ (cf.~\cite{cordoni2022stochastic}) and learn only the nonconservative term, reporting energy relative to this storage. 
When learning dissipation, we parameterize $R_\theta=D_\theta^\top D_\theta$.
Trajectories are generated by simulating the ground-truth SDEs using the Euler-Maruyama method.

\begin{table*}[t]
\centering
\small
\caption{One-step drift error (True-MSE) and mean absolute rollout errors on the deterministic mean dynamics over $t \in [0,20]$ from $(q_0,p_0)=(1,0)$. 
The energy error uses the ground-truth Hamiltonian $H$.}
\begin{tabular}{llcccc}
\hline
System & Model & True-MSE & Mean $|\Delta q|$ & Mean $|\Delta p|$ & Mean $|\Delta H|$\\
\hline
Mass-spring & Baseline & $3.63\times10^{-3}$ & 0.259 & 0.269 & 0.557\\
Mass-spring & SPH-NN-IB & $2.68\times10^{-3}$ & 0.065 & 0.049 & 0.065\\
Mass-spring & SPH-NN-CE & $1.14\times10^{-3}$ & 0.059 & 0.057 & 0.011\\
Mass-spring & SPH-NN-NLL & $2.65\times10^{-3}$ & 0.065 & 0.052 & 0.055\\
\hline
Duffing & Baseline & $5.67\times10^{-3}$ & 0.331 & 0.393 & 0.083\\
Duffing & SPH-NN-IB & $2.60\times10^{-3}$ & 0.018 & 0.017 & 0.005\\
Duffing & SPH-NN-CE & $3.23\times10^{-3}$ & 0.030 & 0.022 & 0.011\\
Duffing & SPH-NN-NLL & $2.76\times10^{-3}$ & 0.020 & 0.039 & 0.007\\
\hline
Van der Pol & Baseline & $2.53\times10^{-3}$ & 0.247 & 0.330 & 0.349\\
Van der Pol & SPH-NN-IB & $2.87\times10^{-5}$ & 0.037 & 0.051 & 0.053\\
Van der Pol & SPH-NN-CE & $2.20\times10^{-4}$ & 0.011 & 0.017 & 0.018\\
Van der Pol & SPH-NN-NLL & $1.50\times10^{-4}$ & 0.015 & 0.020 & 0.023\\
\hline
\end{tabular}
\label{tab:oscillator_results}
\end{table*}

\subsection{Mass-spring oscillator}
Consider the mass–spring system
\begin{equation*}
    m\ddot{x}=-kx+F,
\end{equation*}  
where $x$ is the position of the system, $m$ its mass, $F$ the applied force and $k$ the stiffness of the spring. 
Defining $p=m\dot{x}$ as the momentum and $q=x$. 
$X=(p,q)$ defines a PHS with respect to the energy
\begin{equation*}
    H(p,q) = \frac{1}{2}kq^2+\frac{1}{2}\frac{p^2}{m},
\end{equation*}
of the form
\begin{equation*}
\begin{cases}
\dot{X}&=J\partial_x H(X)+gF,\\
y&=g^\top\partial_x H(X),
\end{cases}
\end{equation*}
with
\begin{equation*}
   J=\begin{pmatrix}0&1\\-1&0\end{pmatrix},
   \quad g=\begin{pmatrix}0\\1\end{pmatrix},
   \quad \partial_x H(X)=\begin{pmatrix}kq\\\frac{p}{m}\end{pmatrix}.
\end{equation*}
We can generalize the system to a SPHS in It\^o form
\begin{equation*}
\begin{pmatrix}dq_t\\dp_t\end{pmatrix}=
\begin{pmatrix}\frac{p_t}{m}-\frac{kq_t}{2m}\\-kq_t-\frac{k}{2m}p_t+F\end{pmatrix}\,dt+
\begin{pmatrix}\frac{p_t}{m}\\-kq_t\end{pmatrix}\,dW_t,
\end{equation*}
where $W_t$ is a standard Brownian motion, and $q_t$ and $p_t$ denote respectively the position and the momentum. In what follows we consider $F=0$ which yields an undamped mass-spring oscillator.

Figure~\ref{fig:phase_space_mass_spring} shows that the baseline drift accumulates phase error and yields a large energy mismatch on long rollouts, while all SPH variants remain close to the invariant orbit.
\begin{figure}[htbp]
    \centering
    \includegraphics[width=0.5\textwidth]{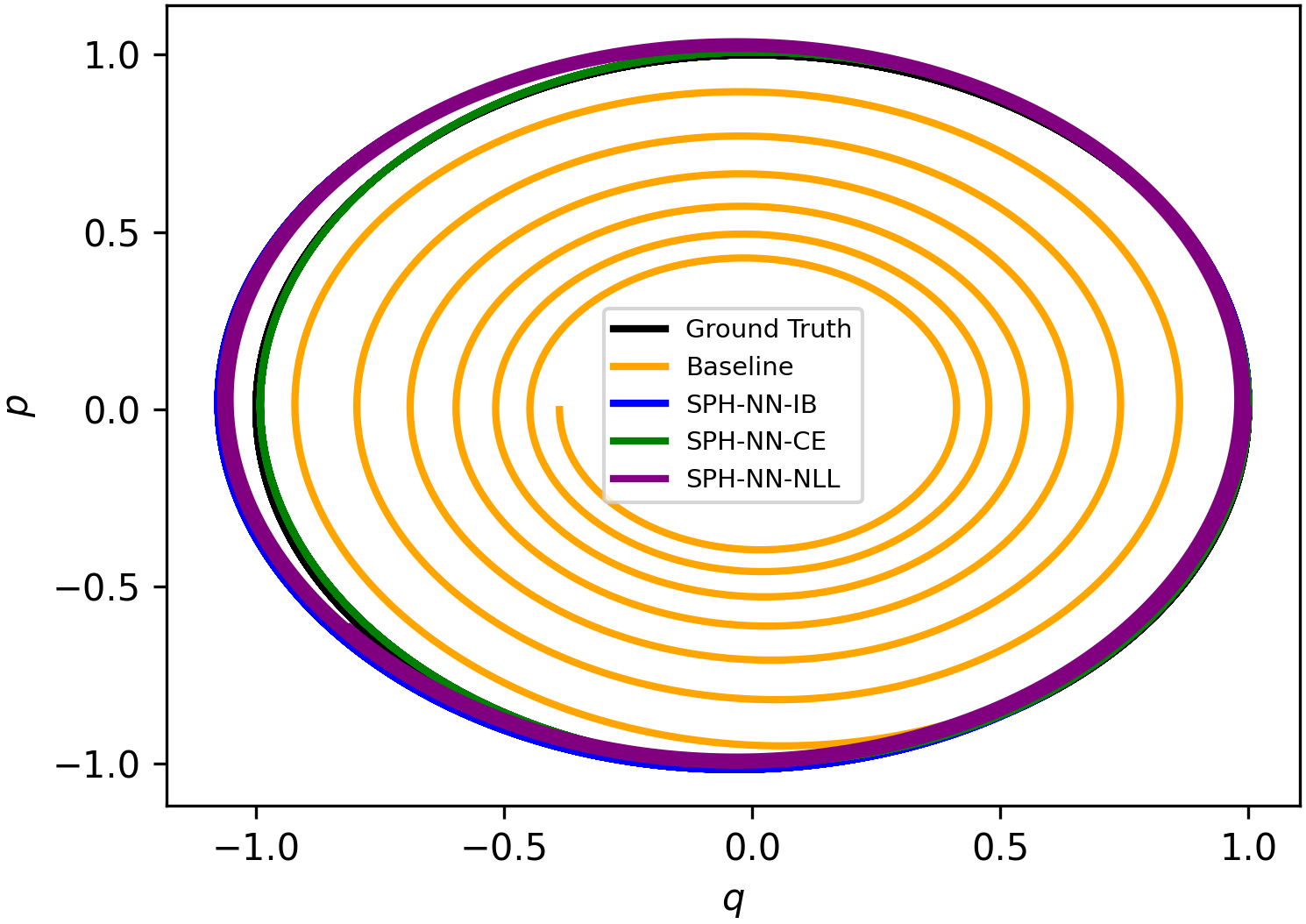}
    \caption{Phase space - Mass spring oscillator.}
    \label{fig:phase_space_mass_spring}
\end{figure}

\begin{figure}[htbp]
    \centering
    \includegraphics[width=0.5\textwidth]{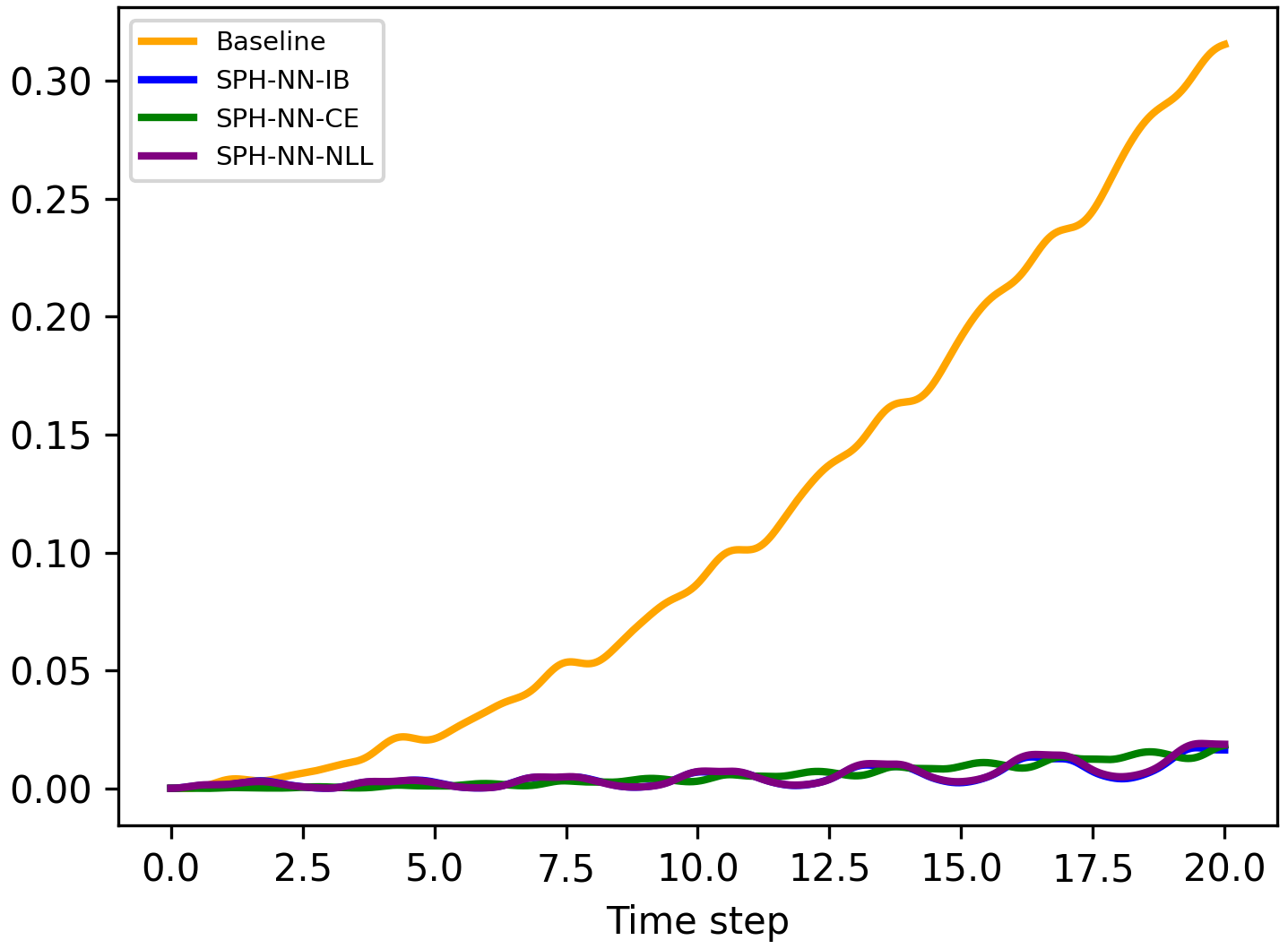}
    \caption{Mean squared error of the rollout - Mass-spring oscillator: the baseline error grows quickly, while the SPH methods keep the error low and stable.}
    \label{fig:mse_rollout_mass_spring}
\end{figure}

\begin{figure}[htbp]
    \centering
    \includegraphics[width=0.5\textwidth]{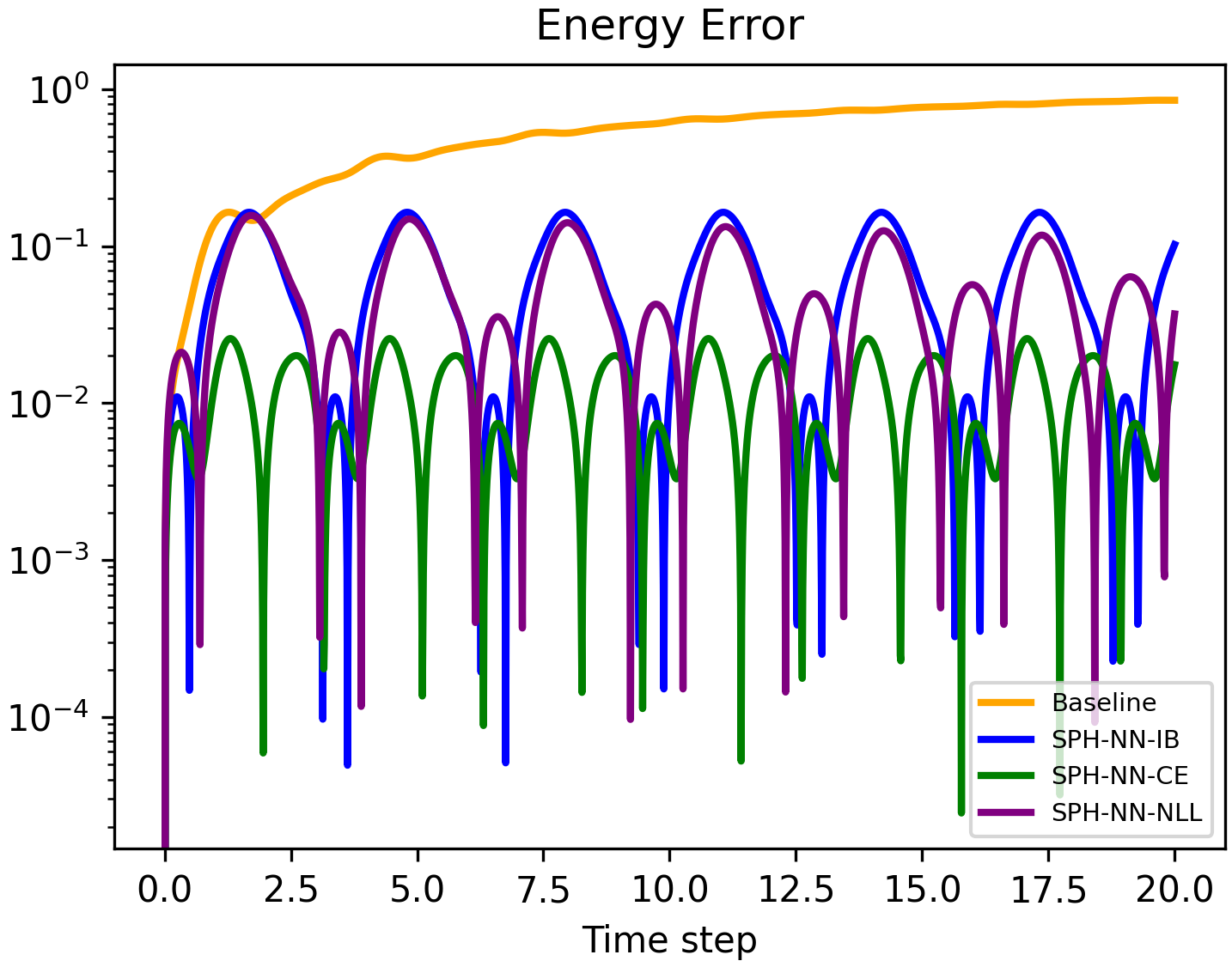}
    \caption{Energy error - Mass-spring oscillator: SPH-NN models significantly reduce energy error compared to the baseline, with the CE loss achieving the most accurate and stable energy behavior over time.}
\end{figure}

\begin{figure}[htbp]
    \centering
    \includegraphics[width=0.5\textwidth]{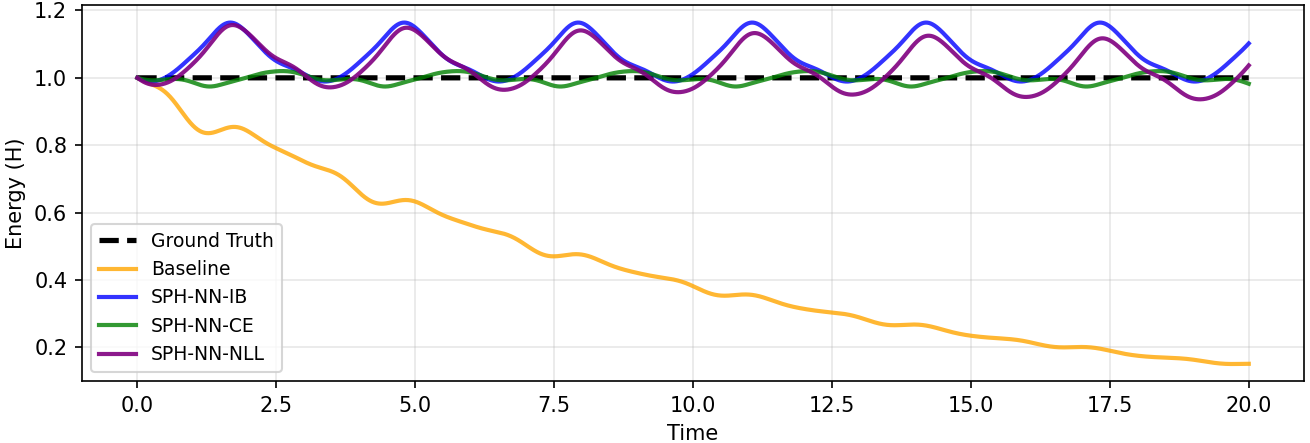}
    \caption{Energy evolution - Mass-spring oscillator: the baseline loses energy steadily, while SPH models keep energy close to the true constant value.}
\end{figure}

\begin{figure}[htbp]
    \centering
    \includegraphics[width=0.5\textwidth]{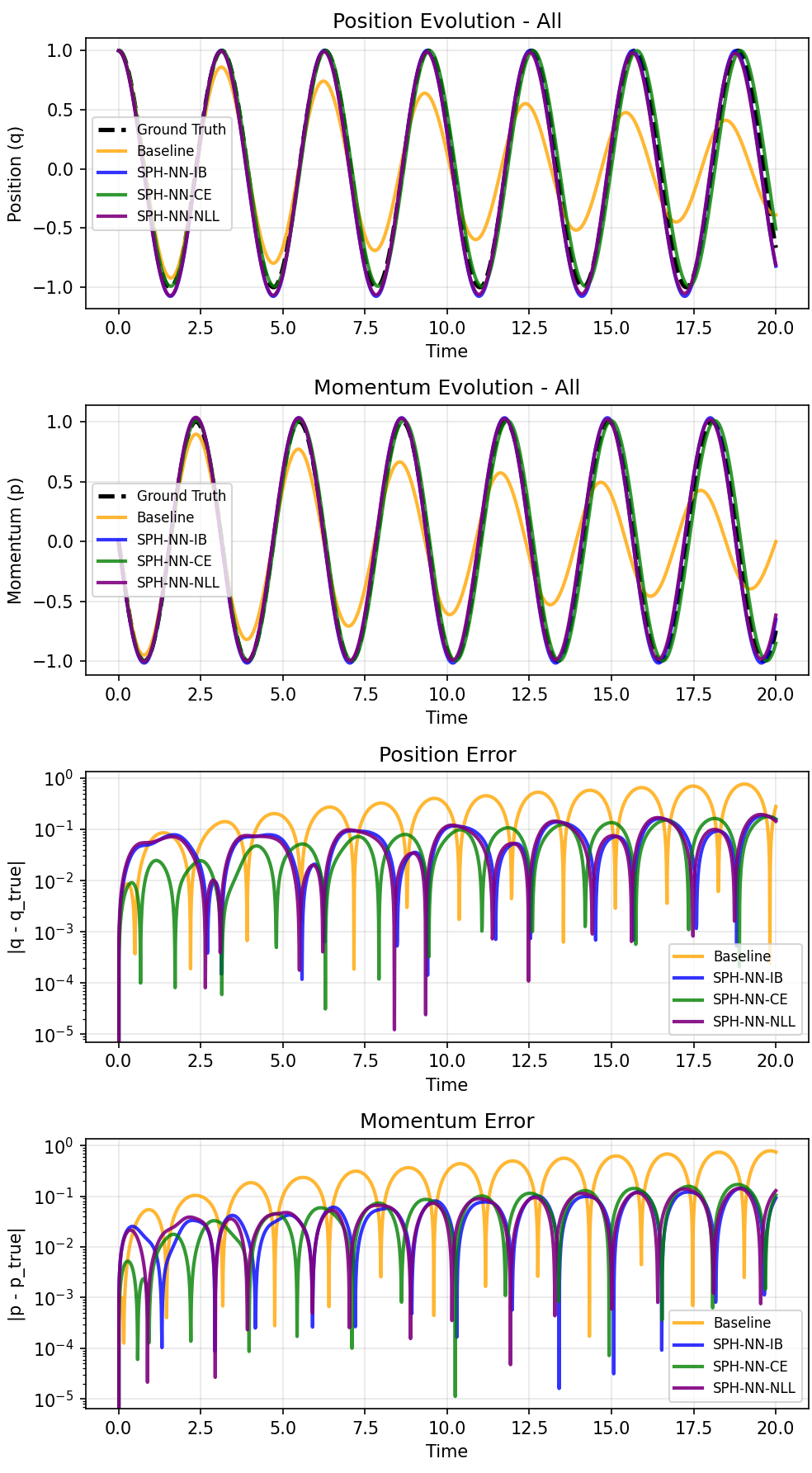}
    \caption{Position and momentum time evolution - Mass-spring oscillator: the baseline drifts and loses amplitude, while SPH methods track the true oscillation and keep errors much smaller.}
\end{figure}

\subsection{Duffing oscillator}
Consider the undamped Duffing system
\begin{equation*}
   \ddot{q}=q-q^3+F,
\end{equation*}
where $q$ is the position and $F$ the applied force.
Defining $p=\dot{q}$ as the momentum and $X=(q,p)$, $X$ defines a PHS with respect to the energy
\begin{equation*}
   H(q,p) = \frac{1}{2}p^2-\frac{1}{2}q^2+\frac{1}{4}q^4,
\end{equation*}
of the form
\begin{equation*}
\begin{cases}
\dot{X}&=J\partial_x H(X)+gF,\\
y&=g^\top\partial_x H(X)
\end{cases}
\end{equation*}
with
\begin{equation*}
  J=\begin{pmatrix}0&1\\-1&0\end{pmatrix},\quad 
  g=\begin{pmatrix}0\\1\end{pmatrix},\quad 
  \partial_x H(X)=\begin{pmatrix}-q+q^3\\p\end{pmatrix}.
\end{equation*}
We can generalize the system to a SPHS in It\^o form
\begin{equation*}
\begin{pmatrix}dq_t\\dp_t\end{pmatrix}=
\begin{pmatrix}p_t\\q_t-q_t^3+F\end{pmatrix}\,dt+
\sigma \,dW_t,
\end{equation*}
where $W_t$ is a standard two dimensional Brownian motion.

Figure~\ref{fig:phase_space_duffing} shows that SPH variants preserve the correct phase-space geometry and reduce long-horizon error growth. Table~\ref{tab:oscillator_results} indicates that SPH-NN-IB is the most accurate in this benchmark, reducing mean rollout errors by more than an order of magnitude relative to the baseline and reducing mean energy error from $0.083$ to $0.005$. 
\begin{figure}[htbp]
    \centering
    \includegraphics[width=0.5\textwidth]{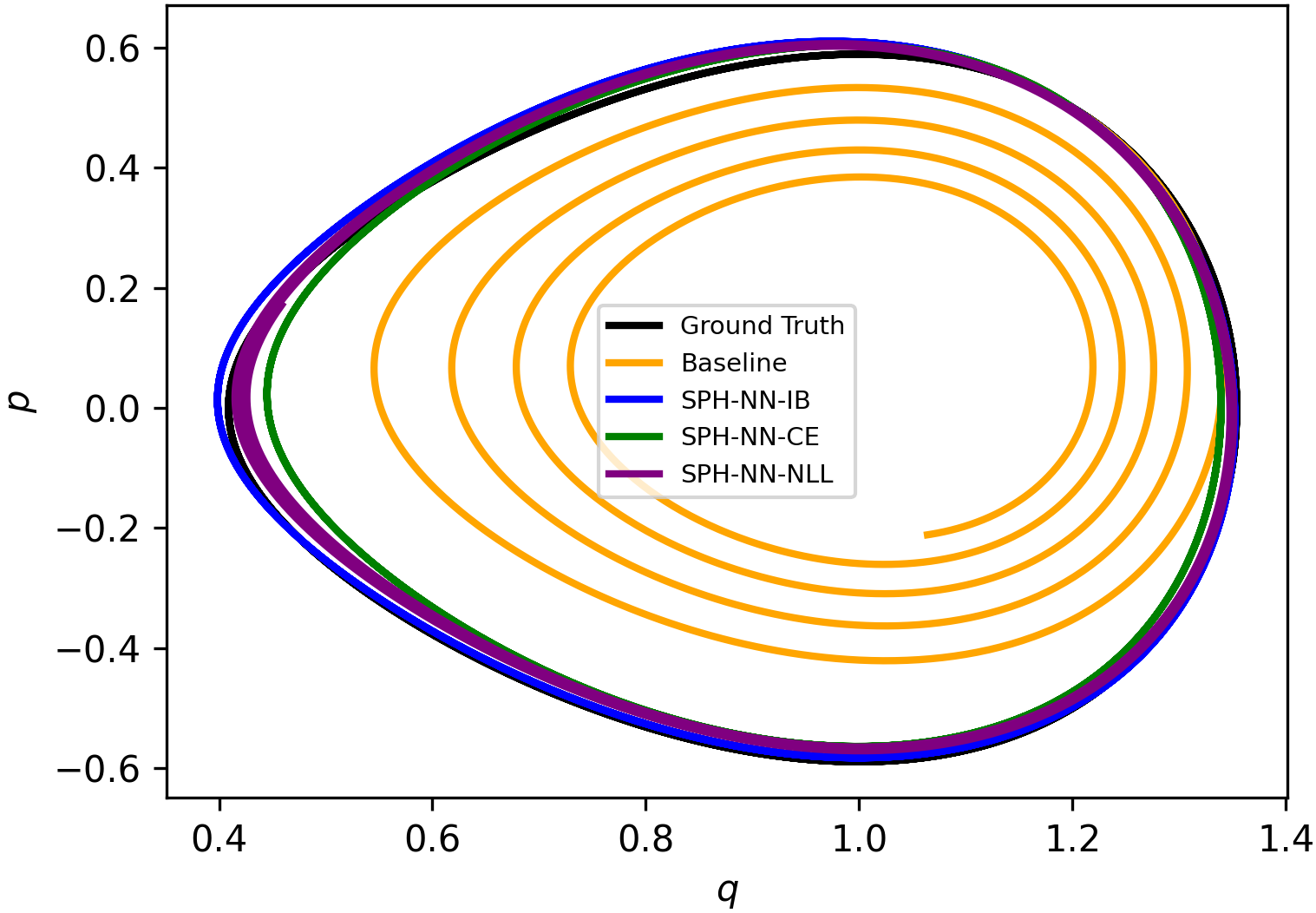}
    \caption{Phase space - Duffing oscillator: the baseline spirals inward, while SPH methods stay close to the true closed trajectory.}
    \label{fig:phase_space_duffing}
\end{figure}

\begin{figure}[htbp]
    \centering
    \includegraphics[width=0.5\textwidth]{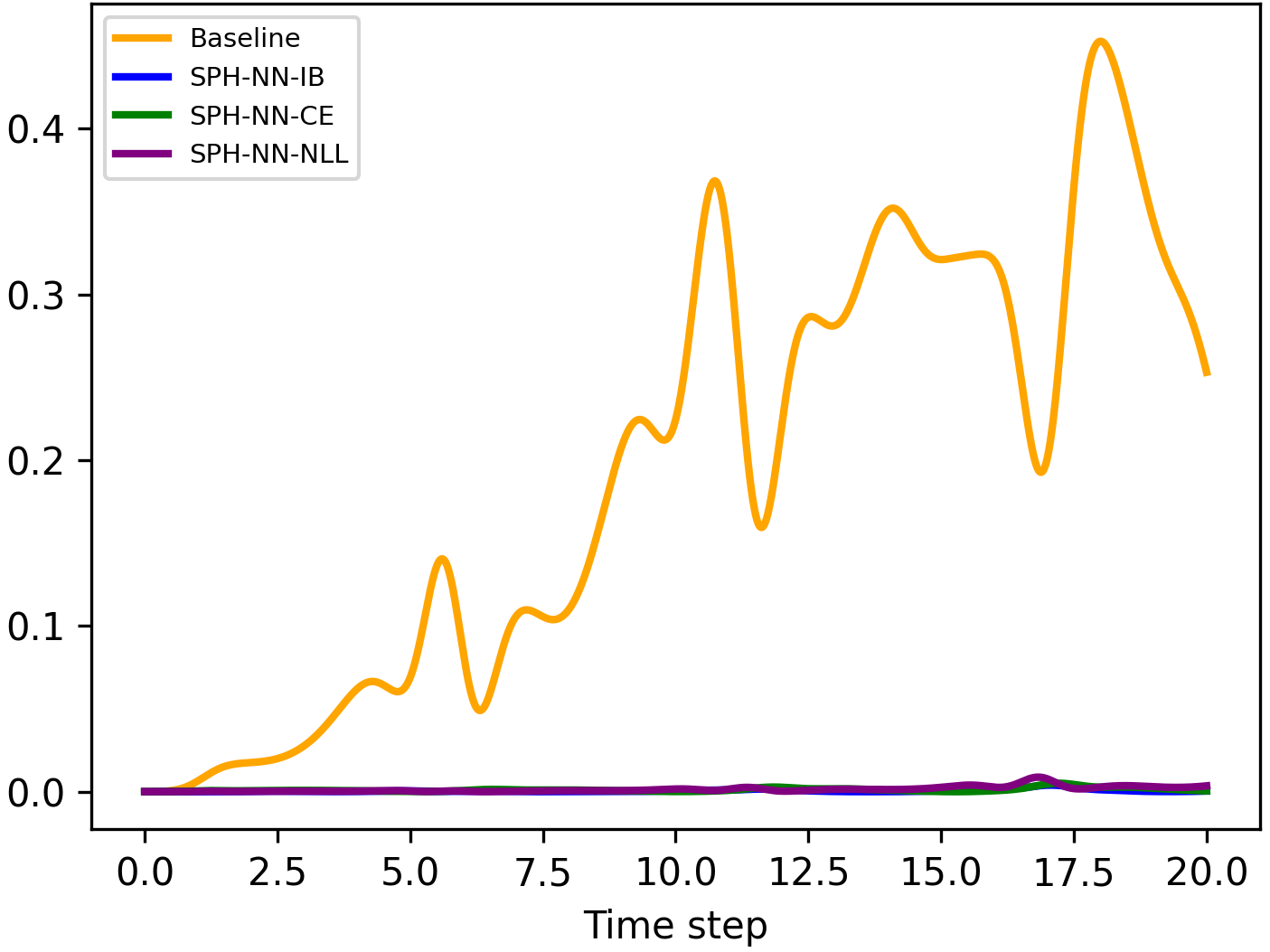}
    \caption{Mean squared error of the rollout - Duffing oscillator.}
\end{figure}

\begin{figure}[htbp]
    \centering
    \includegraphics[width=0.5\textwidth]{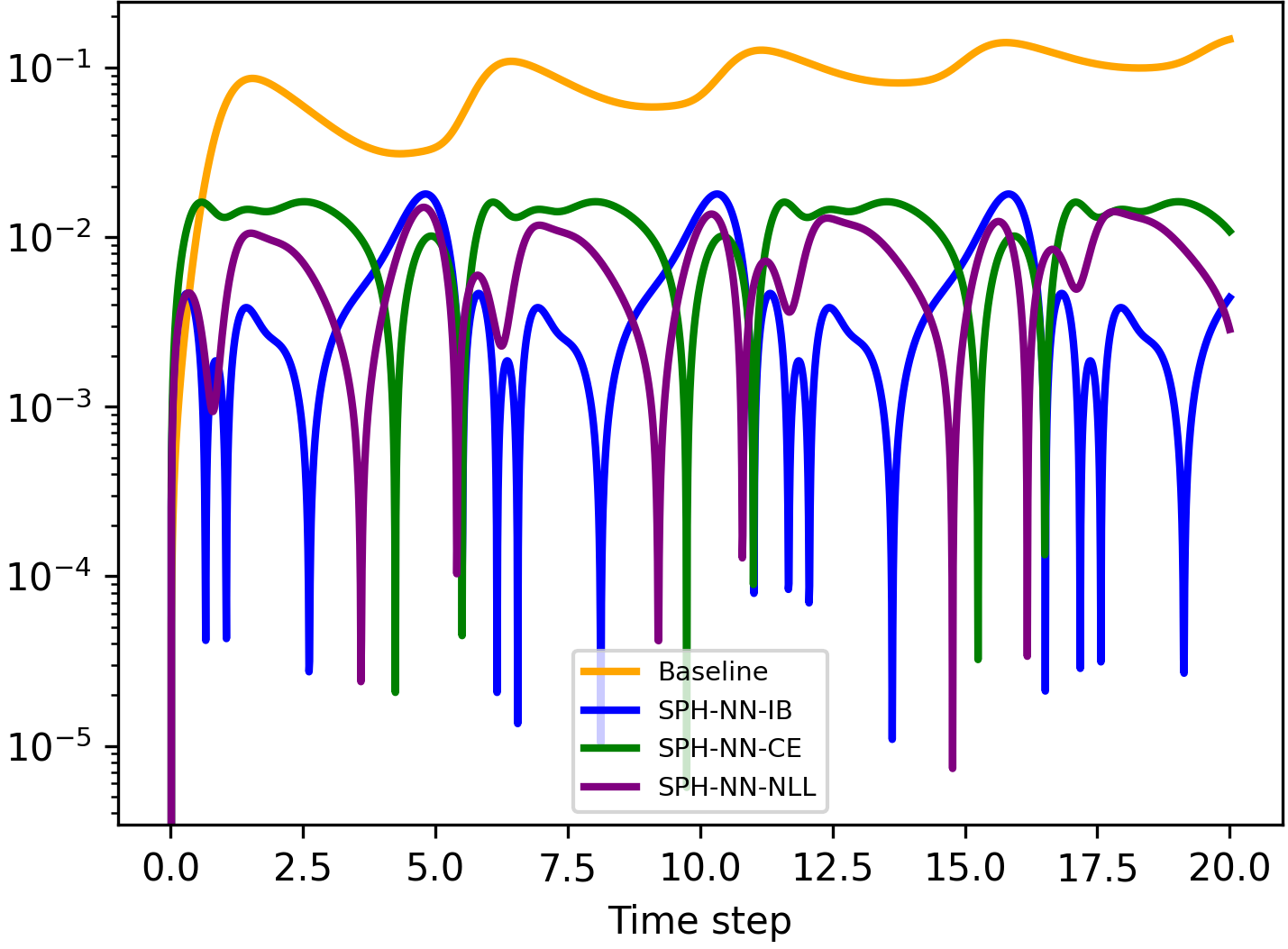}
    \caption{Energy error - Duffing oscillator.}
\end{figure}

\begin{figure}[htbp]
    \centering
    \includegraphics[width=0.5\textwidth]{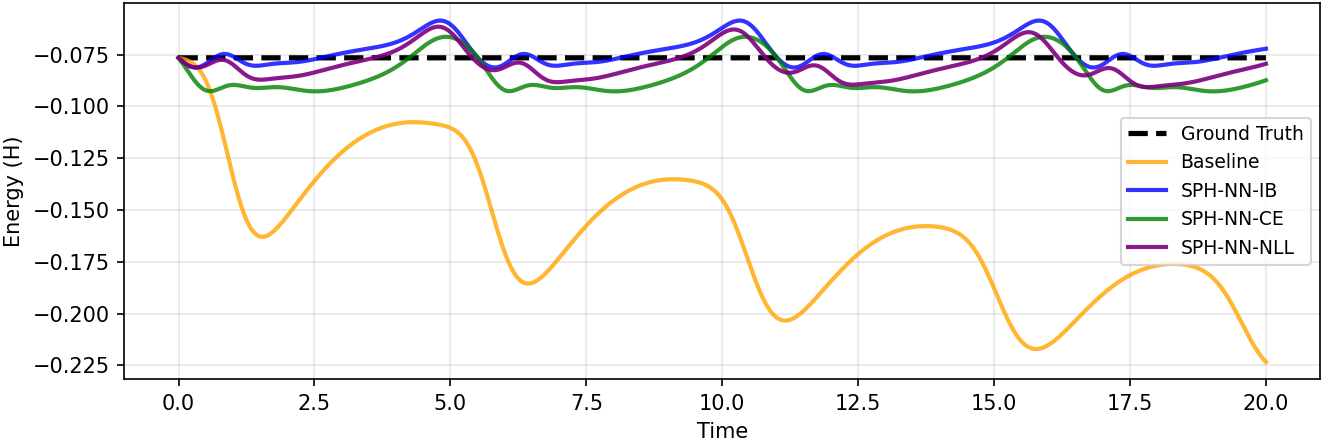}
    \caption{Energy evolution - Duffing oscillator.}
\end{figure}

\begin{figure}[htbp]
    \centering
    \includegraphics[width=0.5\textwidth]{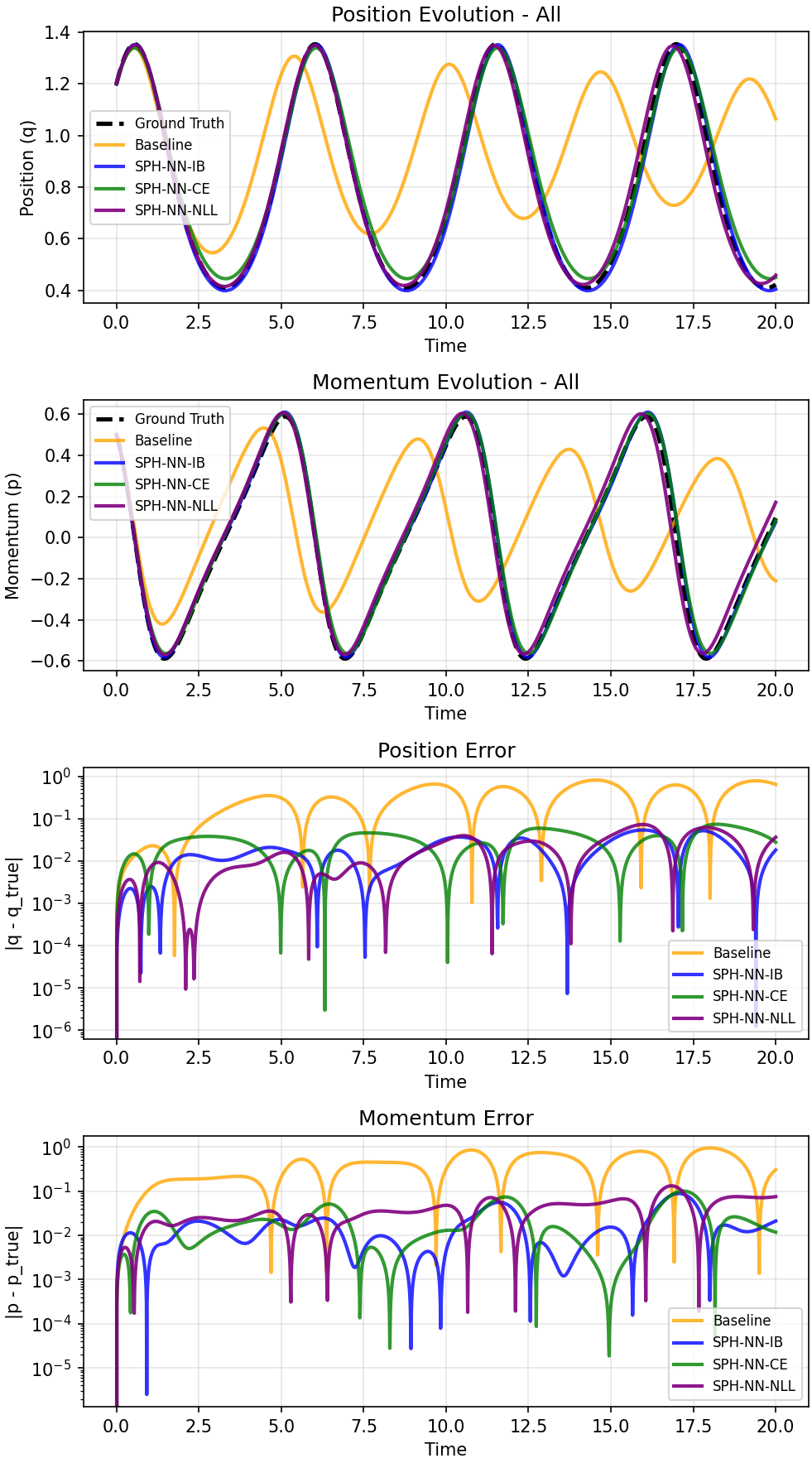}
    \caption{Position and momentum evolution - Duffing oscillator.}
\end{figure}

\subsection{Van der Pol oscillator}
In its standard 2D form, the Van der Pol oscillator is not conservative (hence not canonical Hamiltonian), but it can be written in a (stochastic) port-Hamiltonian form by choosing an energy-like storage function $H$ and encoding the nonconservative dynamics in the dissipation term $R$.
Consider a stochastic Van der Pol oscillator \cite{cordoni2022stochastic}
\begin{equation*}
\begin{cases}
dx_1(t)& = x_2(t) \, dt\\
dx_2(t)& = \Bigl( \mu\bigl(1 - x_1(t)^2\bigr) x_2(t) - x_1(t) \\
&\quad +\frac{1}{2} \xi(x_2(t)) \xi'(x_2(t)) \Bigr) \, dt  + \xi(x_2(t)) \, dW_t.
\end{cases}
\end{equation*}
Here $x(t)=(x_1(t),x_2(t))^\top\in\mathbb{R}^2$ is the state, $t\geq0$ is time, $\mu\in\mathbb{R}$ is the van der Pol parameter, 
$W_t$ is a one-dimensional Brownian motion, $\xi\colon\mathbb{R}\to\mathbb{R}$ is the diffusion coefficient (noise amplitude) acting on the second component.

Since the Van der Pol oscillator is neither conservative nor canonically Hamiltonian in two dimensions, the storage function used in the port-Hamiltonian representation is not uniquely defined. Following \cite{cordoni2022stochastic}, we fix:
\begin{equation*}
   H(X_t)=\frac{1}{2}X_t^\top I X_t,
\end{equation*}
where $X_t$ is the state variable and $I$ is the $2\times2$ identity matrix.

Figure~\ref{fig:phase_space_vdp} shows that the SPH variant learns the limit cycle and reduces long-horizon mismatch, whereas the baseline has phase and amplitude errors.
\begin{figure}[htbp]
    \centering
    \includegraphics[width=0.5\textwidth]{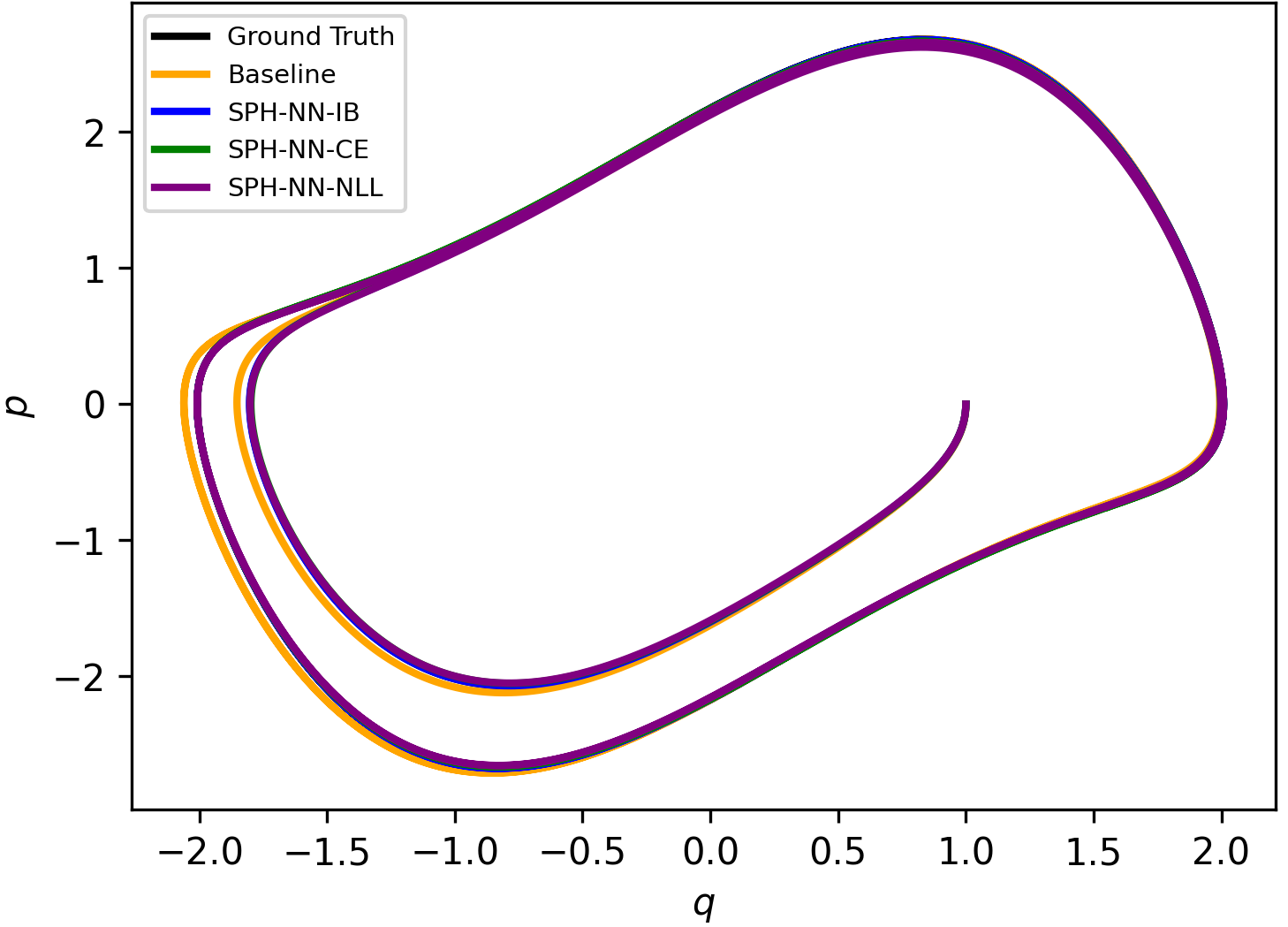}
    \caption{Phase space - Van der Pol.}
    \label{fig:phase_space_vdp}
\end{figure}

\begin{figure}[htbp]
    \centering
    \includegraphics[width=0.5\textwidth]{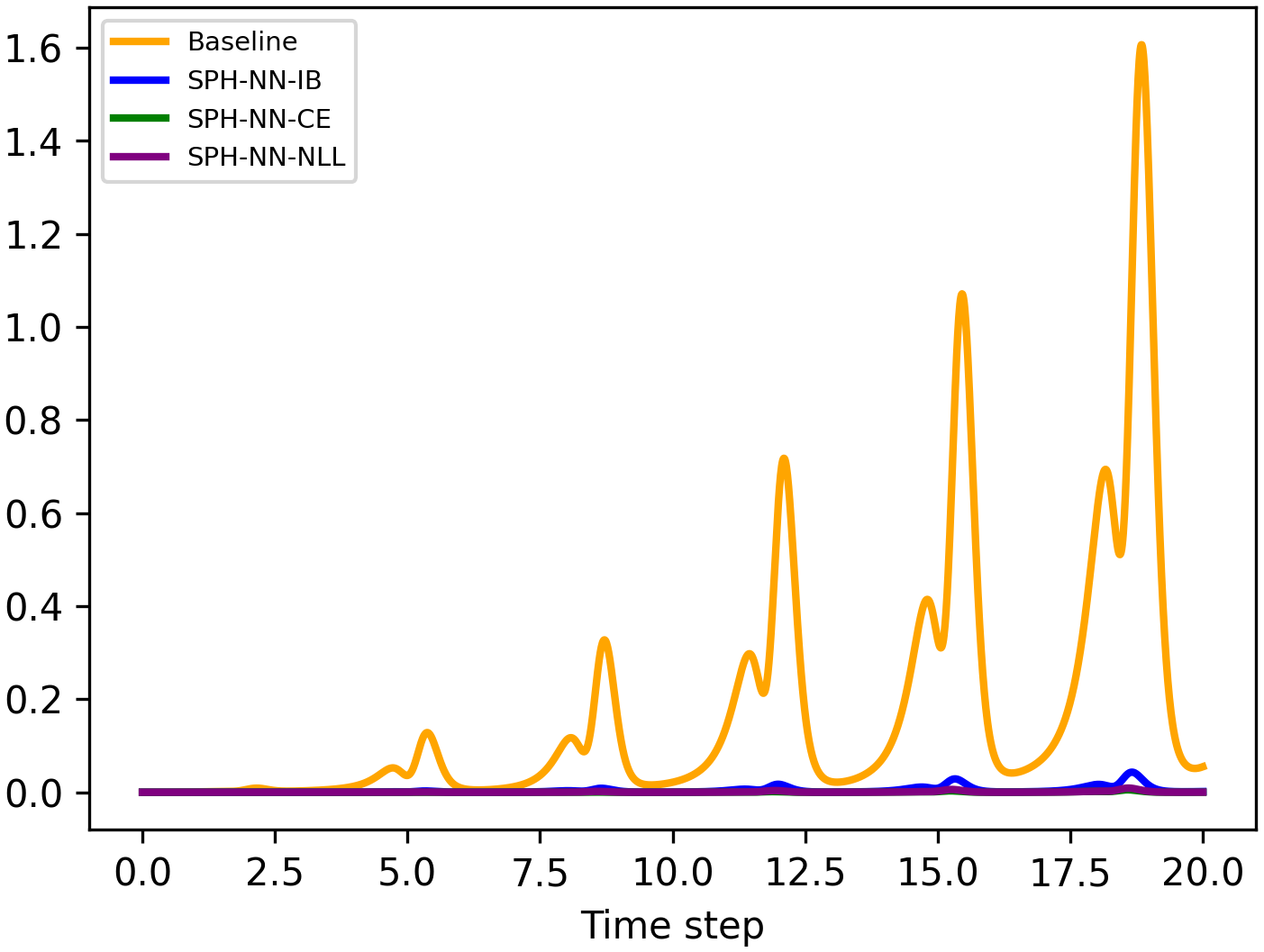}
    \caption{Mean squared error of the rollout - Van der Pol: the baseline error explodes with large spikes, while SPH methods keep the error low and stable across the rollout.}
\end{figure}

\begin{figure}[htbp]
    \centering
    \includegraphics[width=0.5\textwidth]{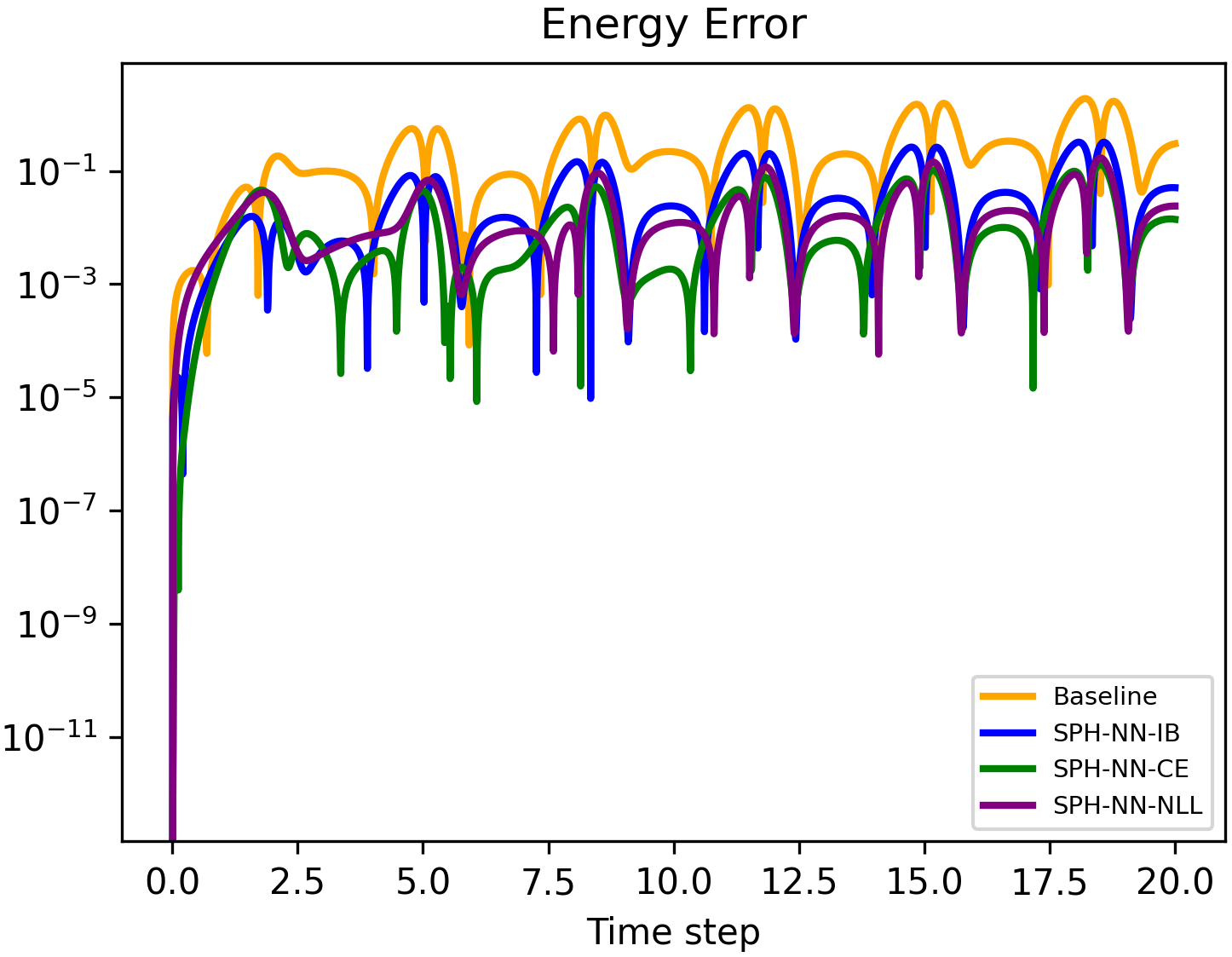}
    \caption{Energy error - Van der Pol oscillator.}
\end{figure}

\begin{figure}[htbp]
    \centering
    \includegraphics[width=0.5\textwidth]{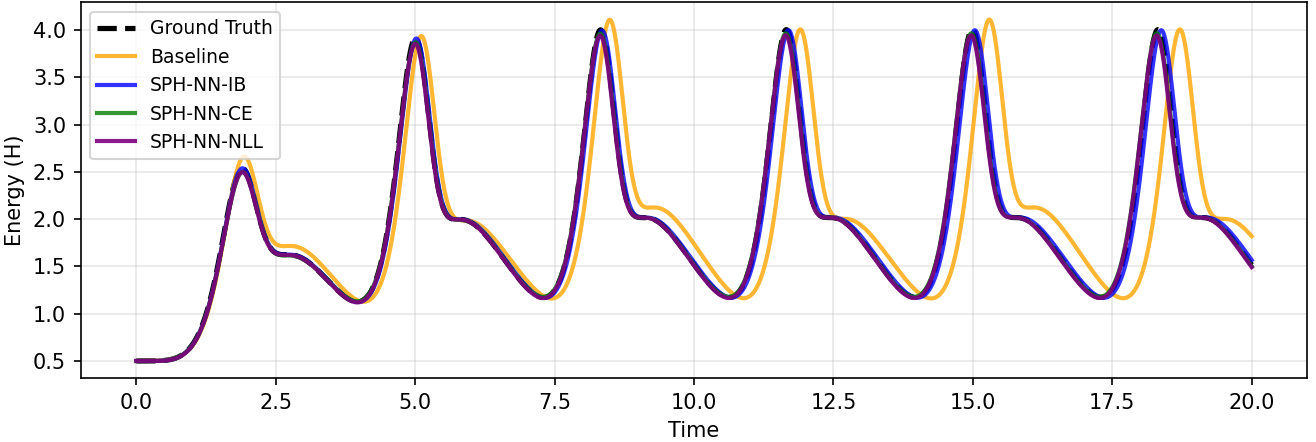}
    \caption{Energy evolution - Van der Pol oscillator: SPH methods match the true energy oscillations closely, while the baseline slowly drifts and becomes misaligned.}
    \label{fig:energy_evolution_vdp}
\end{figure}

\begin{figure}[htbp]
    \centering
    \includegraphics[width=0.5\textwidth]{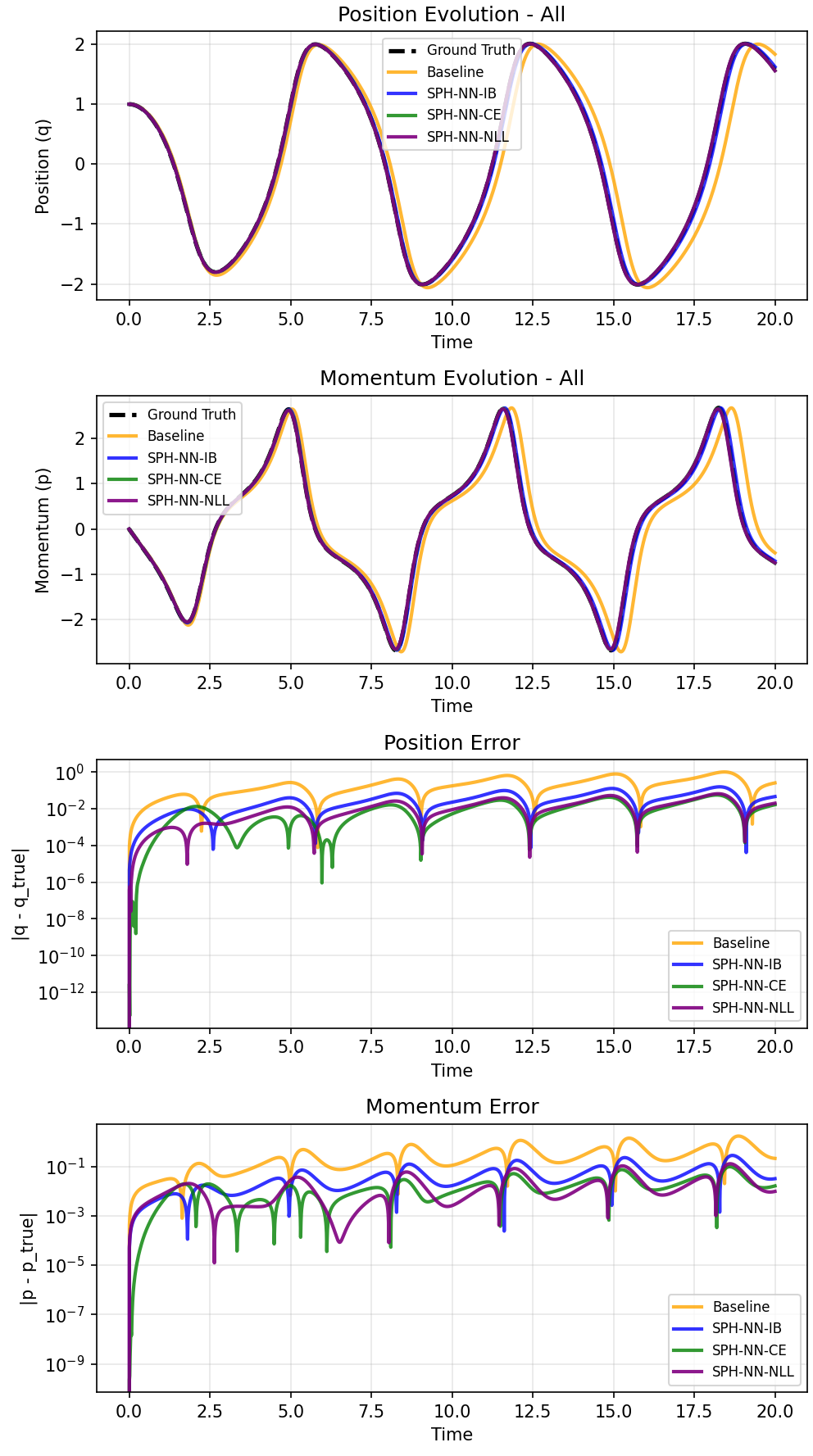}
    \caption{Position and momentum evolution - Van der Pol oscillator.}
    \label{fig:q_p_vdp}
\end{figure}


\section{Universal Approximation Theorem of SPHS}
In this section, we prove an approximation theorem on compact domains and finite horizons. 
More precisely, for any compact set $\mathcal{K}$ and $T > 0$, the SPH-NN parameterization can approximate the coefficients of a given SPHS uniformly on $\mathcal{K}$ (including $C^2$-accuracy of the Hamiltonian). 
Moreover, under a natural coupling driven by the same Brownian motion, the corresponding solution processes remain close up to the joint exit time from $\mathcal{K}$ (Theorem~\ref{thm:Structured_UAT_SPHS}). 
We then derive weak-passivity bounds in expectation via generator inequalities (Corollary~\ref{cor:Weak_passivity}).

\begin{thm}[\textbf{UAT for SPHS}]\label{thm:Structured_UAT_SPHS}
Fix $T > 0$. 
Let $\mathcal{K} \subset \mathbb{R}^n$ and $U \subset\mathbb{R}^m$ be compact, 
and let $u \in C([0,T],U)$. 
Let $g\colon\mathbb{R}^n \to \mathbb{R}^{n \times m}$ be continuous, locally Lipschitz and assumed known. 
Let $J,R\colon\mathbb{R}^n \to \mathbb{R}^{n \times n}$ and $\sigma\colon\mathbb{R}^n \to \mathbb{R}^{n\times d}$ be continuous, locally Lipschitz, and let $H \in C^2(\mathbb{R}^n)$.

Assume that for all $x \in \mathcal{K}$,
\begin{equation*}
  J(x) = - J(x)^\top,\; R(x)=R(x)^\top\succeq 0,
\end{equation*}
and assume there exist $r\in\mathbb{N}$ and a continuous map $D\colon\mathcal{K}\to\mathbb{R}^{r\times n}$ such that $R(x) = D(x)^\top D(x)$ for all $x \in \mathcal{K}$.
Assume the activation $G$ is $l$-finite in the sense of \cite{hornik1990universal} for some $l \in \mathbb{N}$ with $l > 2$.

Let $X$ be a strong solution on $[0,T]$ of
\begin{equation*}
\begin{aligned}
dX_t &= \Big(\bigl[J(X_t)-R(X_t)\bigr]\nabla H(X_t)+g(X_t)u_t\Big)\,dt \\
&\quad + \sigma(X_t)\,dW_t,
\qquad X_0=x_0\in\mathcal{K}.
\end{aligned}
\end{equation*}

Then, for any $\varepsilon>0$ and $\delta>0$ there exist feedforward neural networks
\begin{gather*}
\hat{A}\colon\mathbb{R}^n \to \mathbb{R}^{n \times n},\quad
\hat{D}\colon\mathbb{R}^n \to \mathbb{R}^{r \times n}, \\
\hat{H}\colon\mathbb{R}^n \to \mathbb{R},\quad
\hat{\sigma}\colon\mathbb{R}^n \to \mathbb{R}^{n \times d},
\end{gather*}
and defining
\begin{align*}
\mathcal{N}_\theta^J(x) &= \hat{A}(x)-\hat{A}(x)^\top,\quad
\mathcal{N}_\theta^R(x) = \hat{D}(x)^\top\hat{D}(x), \\
\mathcal{N}_\theta^H(x) &= \hat{H}(x),\quad
\mathcal{N}_\theta^\sigma(x) = \hat{\sigma}(x).
\end{align*}
The following hold:
\begin{enumerate}
\item $\mathcal{N}_\theta^J(x) = - (\mathcal{N}_\theta^J (x))^\top$ 
     and $\mathcal{N}_\theta^R(x) \succeq 0$ for all $x \in \mathbb{R}^n$, and
\begin{equation*}
\begin{split}
& \sup_{x \in \mathcal{K}}\Big(
\|\mathcal{N}_\theta^J(x)-J(x)\|
+ \|\mathcal{N}_\theta^R(x)-R(x)\| \\
& \quad + \|\mathcal{N}_\theta^\sigma(x)-\sigma(x)\|
\Big) + \|\mathcal{N}_\theta^H-H\|_{C^2(\mathcal{K})} \le\varepsilon.
\end{split}
\end{equation*}

\item Let $\widehat{X}$ be a strong solution on $[0,T]$ of
\begin{equation*}
\begin{aligned}
d\widehat{X}_t
&= \big((\mathcal{N}_\theta^J(\widehat{X}_t)-\mathcal{N}_\theta^R(\widehat{X}_t))
\nabla\mathcal{N}_\theta^H(\widehat{X}_t)+g(\widehat{X}_t)u_t\big)\,dt \\
&\quad + \mathcal{N}_\theta^\sigma(\widehat{X}_t)\,dW_t,
\quad \widehat{X}_0=x_0,
\end{aligned}
\end{equation*}
driven by the same $W$, and define
\begin{equation*}
\tau_{\mathcal{K}} = \inf\{t \geq 0\colon X_t \notin \mathcal{K}\text{ or }\widehat{X}_t \notin \mathcal{K}\}.
\end{equation*}
Then we have
\begin{equation*}
\mathbb{P}\Big(\sup_{0 \le t \le T\wedge\tau_{\mathcal{K}}}|X_t-\widehat{X}_t| \le \varepsilon\Big)\ge 1 - \delta.
\end{equation*}
\end{enumerate}
\end{thm}

\begin{pf}
Choose $\eta > 0$ such that $\eta \leq \varepsilon$ and $\eta \leq 1$ (specified below). 
Apply Lemma~\ref{lem:PH_structured_approx} on $\mathcal{K}$ with tolerance $\eta$ to $(J, D, \sigma, H)$ to obtain networks $\hat{A}$, $\hat{D}$, $\hat{H}$, $\hat{\sigma}$ such that item (1) holds with $\eta$ in place of $\varepsilon$, hence with $\varepsilon$ since $\eta \leq \varepsilon$.

Define
\begin{align*}
b(t,x)
&= (J(x)-R(x))\nabla H(x)+g(x)u_t, \\
\hat{b}(t,x)
&= (\mathcal{N}_\theta^J(x)-\mathcal{N}_\theta^R(x))
\nabla\mathcal{N}_\theta^H(x)+g(x)u_t.
\end{align*}
(Here $\|\cdot\|$ denotes the operator norm and $\|\cdot\|_F$ denotes the Frobenius norm.)

Since $J$, $R$, $g$, $\sigma$ are locally Lipschitz and $\nabla H$ is locally Lipschitz, their restrictions to the compact set $\mathcal{K}$ are Lipschitz and bounded. 
Since $u$ is bounded on $[0,T]$, there exists $L < \infty$ such that for all $x,y \in \mathcal{K}$ and $t \in [0,T]$,
\begin{equation*}
    |b(t,x) - b(t,y)| + \|\sigma(x) - \sigma(y)\|_F\le L|x-y|.
\end{equation*}
Set
\begin{equation*}
\alpha = \sup_{t \in [0,T]}\sup_{x \in \mathcal{K}}|b(t,x) - \hat{b}(t,x)|,\;
\beta = \sup_{x \in \mathcal{K}}\|\sigma(x) - \mathcal{N}_\theta^\sigma(x)\|_F.
\end{equation*}
Let
\begin{equation*}
M_{JR} = \sup_{x \in \mathcal{K}}\|J(x) - R(x)\|, \; 
M_1 = \sup_{x \in \mathcal{K}}\|\nabla H(x)\|.
\end{equation*}
For $x \in \mathcal{K}$, we obtain
\begin{equation*}
\begin{split}
|b(t,x)-\hat{b}(t,x)| &\leq \|J(x)-R(x)\|
\|\nabla H(x)-\nabla\mathcal{N}_\theta^H(x)\| \\
&+ \bigl(\|\mathcal{N}_\theta^J(x)-J(x)\|
+ \|\mathcal{N}_\theta^R(x)-R(x)\|\bigr) \\
&\quad \times \|\nabla\mathcal{N}_\theta^H(x)\|.
\end{split}
\end{equation*}

By Lemma~\ref{lem:PH_structured_approx} (with tolerance $\eta$), on $\mathcal{K}$ we have
\begin{gather*}
\|\nabla H(x)-\nabla\mathcal{N}_\theta^H(x)\| \leq \eta,\\
\|\mathcal{N}_\theta^J(x)-J(x)\|+\|\mathcal{N}_\theta^R(x)-R(x)\|\le\eta.
\end{gather*}
Also, using Lemma~\ref{lem:PH_structured_approx} and $\eta \le1$ gives
$\|\nabla \mathcal{N}_\theta^H(x)\| \leq \|\nabla H(x)\| + \eta \le M_1 + 1$, and hence
\begin{equation*}
\alpha \le(M_{JR}+M_1+1) \eta.
\end{equation*}
Moreover, Lemma~\ref{lem:PH_structured_approx} gives $\|\sigma(x)-\mathcal{N}_\theta^\sigma(x)\|\leq \eta$ on $\mathcal{K}$, and $\|A\|_F \le\sqrt{d} \|A\|$ for $A \in \mathbb{R}^{n \times d}$, so $\beta \le\sqrt{d}\eta$. 
Let
\begin{equation*}
C_1 = M_{JR} + M_1+1+\sqrt{d},
\end{equation*}
so $\alpha + \beta \leq C_1\eta$.

We apply Proposition~\ref{prop:Stopped_SDE_stability} to get
\begin{equation*}
\begin{aligned}
& \mathbb{P}\Big(
\sup_{0 \le t \le T\wedge\tau_{\mathcal{K}}}
|X_t-\widehat{X}_t| > \varepsilon
\Big) \\
&\leq \frac{(4T^2+16T)\exp((4T+16)L^2T)}{\varepsilon^2} \; \times (C_1\eta)^2.
\end{aligned}
\end{equation*}

Choose $\eta > 0$ such that $\eta \leq \varepsilon$, $\eta \leq 1$, and
\begin{equation*}
(4T^2+16T)\exp((4T+16)L^2T)(C_1\eta)^2 \leq \delta\varepsilon^2.
\end{equation*}
Then the probability above is at most $\delta$, which gives item (2).
\end{pf}

\begin{cor}[\textbf{Localized weak passivity}] \label{cor:Weak_passivity}
Assume the hypotheses of Theorem~\ref{thm:Structured_UAT_SPHS}. Define
\begin{equation*}
  y(x) = g(x)^\top \nabla H(x)
\end{equation*}
and assume there exists $c_0 \geq 0$ such that for all $x \in \mathcal{K}$,
\begin{multline*}
r(x) = \\
\frac{1}{2}\Tr(\sigma(x) \sigma(x)^\top\nabla^2H(x)) - \nabla H(x)^\top R(x)\nabla H(x) \leq c_0.
\end{multline*}
Fix $\varepsilon > 0$ and choose the networks from Theorem~\ref{thm:Structured_UAT_SPHS} with a coefficient tolerance $\eta > 0$ small enough. Next, define
\begin{equation*}
   \hat{y}(x) = g(x)^\top \nabla\mathcal{N}_\theta^H(x),
\end{equation*}
\begin{equation*}
\begin{split}
\hat{r}(x)
&= \frac{1}{2}\Tr\Big(
\mathcal{N}_\theta^\sigma(x)(\mathcal{N}_\theta^\sigma(x))^\top
\nabla^2\mathcal{N}_\theta^H(x)
\Big) \\
&\quad - \nabla\mathcal{N}_\theta^H(x)^\top
\mathcal{N}_\theta^R(x) \; \times \nabla\mathcal{N}_\theta^H(x).
\end{split}
\end{equation*}

Then:
\begin{enumerate}
\item $\hat{r}(x) \leq c_0 + \varepsilon$ for all $x \in \mathcal{K}$.
\item Let $\widehat{X}$ be the learned solution from Theorem~\ref{thm:Structured_UAT_SPHS} and let
\begin{equation*}
\tau_{\mathcal{K}}^{\widehat{X}} = \inf\{t \ge0 \colon\widehat{X}_t \notin \mathcal{K}\}.
\end{equation*}
Then for all $t \in [0,T]$,
\begin{equation*}
\begin{split}
\mathbb{E}\big[\mathcal{N}_\theta^H(\widehat{X}_{t\wedge\tau_{\mathcal{K}}^{\widehat{X}}})\big]
&\leq \mathcal{N}_\theta^H(x_0)
+ \mathbb{E}\int_0^{t\wedge\tau_{\mathcal{K}}^{\widehat{X}}}
u_s^\top\hat{y}(\widehat{X}_s)\,ds \\
&\quad + (c_0+\varepsilon)t.
\end{split}
\end{equation*}

\end{enumerate}
\end{cor}

\begin{pf}
Let $\eta$ be the coefficient tolerance on $\mathcal{K}$ from Theorem~\ref{thm:Structured_UAT_SPHS}, so on $\mathcal{K}$,
\begin{gather*}
\|\mathcal{N}_\theta^\sigma-\sigma\| \le\eta,\quad
\|\mathcal{N}_\theta^R-R\| \le\eta, \\
\|\nabla\mathcal{N}_\theta^H-\nabla H\| \le\eta,\quad
\|\nabla^2\mathcal{N}_\theta^H-\nabla^2 H\| \le\eta.
\end{gather*}
Define the finite constants
\begin{gather*}
M_\sigma=\sup_{x\in\mathcal{K}}\|\sigma(x)\|,\quad
M_R=\sup_{x\in\mathcal{K}}\|R(x)\|, \\
M_1=\sup_{x\in\mathcal{K}}\|\nabla H(x)\|,\quad
M_2=\sup_{x\in\mathcal{K}}\|\nabla^2 H(x)\|.
\end{gather*}

Then on $\mathcal{K}$,
\begin{gather*}
\|\mathcal{N}_\theta^\sigma\| \leq M_\sigma+\eta,\quad
\|\mathcal{N}_\theta^R\| \leq M_R+\eta, \\
\|\nabla\mathcal{N}_\theta^H\| \leq M_1+\eta,\quad
\|\nabla^2\mathcal{N}_\theta^H\| \leq M_2+\eta.
\end{gather*}

Moreover,
\begin{equation*}
\begin{split}
\|\mathcal{N}_\theta^\sigma(\mathcal{N}_\theta^\sigma)^\top-\sigma\sigma^\top\|
&\leq (\|\mathcal{N}_\theta^\sigma\|+\|\sigma\|)
\|\mathcal{N}_\theta^\sigma-\sigma\| \\
&\leq (2M_\sigma+\eta)\eta \; \text{ on }\mathcal{K}.
\end{split}
\end{equation*}

Using $|\Tr(AB)| \le n\|A\|\|B\|$, for $x \in\mathcal{K}$,
\begin{equation*}
\begin{split}
|\hat{r}(x)-r(x)| &\leq \frac{1}{2}\Big|
\Tr\big((\mathcal{N}_\theta^\sigma(\mathcal{N}_\theta^\sigma)^\top-\sigma\sigma^\top)
\nabla^2\mathcal{N}_\theta^H\big)\Big| \\
&\quad + \frac{1}{2}\Big|
\Tr\big(\sigma\sigma^\top(\nabla^2\mathcal{N}_\theta^H-\nabla^2 H)\big)\Big| \\
&\quad + \Big|
\nabla\mathcal{N}_\theta^H{}^\top\mathcal{N}_\theta^R\nabla\mathcal{N}_\theta^H
-\nabla H^\top R\nabla H
\Big|.
\end{split}
\end{equation*}

Also, on $\mathcal{K}$,
\begin{multline*}
\Big|\nabla\mathcal{N}_\theta^H{}^\top\mathcal{N}_\theta^R\nabla\mathcal{N}_\theta^H
-\nabla H^\top R\nabla H\Big| \\
\leq \eta\|\nabla\mathcal{N}_\theta^H\|^2 + \|R\|\|\nabla\mathcal{N}_\theta^H-\nabla H\|
\|\nabla\mathcal{N}_\theta^H+\nabla H\|.
\end{multline*}
Combining these bounds and using $\|\nabla\mathcal{N}_\theta^H\| \leq M_1 + \eta$ and $\|\nabla^2 \mathcal{N}_\theta^H\| \leq M_2 + \eta$ yields
\begin{equation*}
\sup_{x \in \mathcal{K}}|\hat{r}(x)-r(x)| \leq C_2 \eta,
\end{equation*}
for some finite constant $C_2$ depending only on $n$, $M_\sigma$, $M_R$, $M_1$, $M_2$. 
Choose $\eta$ small enough so that $C_2 \eta \leq \varepsilon$. 
Since $r(x) \leq c_0$ on $\mathcal{K}$, we get $\hat{r}(x) \leq c_0 + \varepsilon$ on $\mathcal{K}$, proving item (1).

Let $\hat{b}(t,x) = (\mathcal{N}_\theta^J(x) - \mathcal{N}_\theta^R(x)) \nabla\mathcal{N}_\theta^H(x) + g(x)u_t$. 
For $x \in \mathcal{K}$, the generator applied to $\mathcal{N}_\theta^H$ is given by
\begin{equation*}
\begin{split}
(\hat{\mathcal{L}}_t[\mathcal{N}_\theta^H])(x)
&= \nabla\mathcal{N}_\theta^H(x)^\top\hat{b}(t,x) \\
& \quad+\frac{1}{2} \Tr\bigl(\mathcal{N}_\theta^\sigma(x)(\mathcal{N}_\theta^\sigma(x))^\top\nabla^2\mathcal{N}_\theta^H(x)\bigr) \\
&= \nabla\mathcal{N}_\theta^H(x)^\top(\mathcal{N}_\theta^J ( x ) -\mathcal{N}_\theta^R(x)) \nabla \mathcal{N}_\theta^H(x) \\
&+ u_t^\top \hat{y}(x) +\frac{1}{2} \Tr\bigl(\mathcal{N}_\theta^\sigma (x)(\mathcal{N}_\theta^\sigma (x))^\top\nabla^2 \mathcal{N}_\theta^H(x)\bigr).
\end{split}
\end{equation*}
Since $\mathcal{N}_\theta^J(x)$ is skew-symmetric, we have 
\begin{equation*}
\nabla \mathcal{N}_\theta^H(x)^\top\mathcal{N}_\theta^J (x) \nabla\mathcal{N}_\theta^H (x) = 0,
\end{equation*}
hence
\begin{equation*}
(\hat{\mathcal{L}}_t[\mathcal{N}_\theta^H])(x) = u_t^\top \hat{y}(x) + \hat{r}(x) \leq u_t^\top\hat{y}(x) + c_0 + \varepsilon \quad \text{for } x \in \mathcal{K}.
\end{equation*}
Apply It\^o's formula to $\mathcal{N}_\theta^H(\widehat{X}_{t\wedge \tau_{\mathcal{K}}^{\widehat{X}}})$. 
Since the process is stopped on $\mathcal{K}$ and the coefficients are continuous,
the stochastic integral has zero expectation. Taking expectations gives
\begin{equation*}
\mathbb{E} [\mathcal{N}_\theta^H(\widehat{X}_{t \wedge \tau_{\mathcal{K}}^{\widehat{X}}})]
=\mathcal{N}_\theta^H (x_0) + \mathbb{E} \int_0^{t \wedge\tau_{\mathcal{K}}^{\widehat{X}}}(\hat{\mathcal{L}}_s[\mathcal{N}_\theta^H])(\widehat{X}_s) \, ds.
\end{equation*}
For $s < \tau_{\mathcal{K}}^{\widehat{X}}$, $\widehat{X}_s \in \mathcal{K}$, so $(\hat{\mathcal{L}}_s[\mathcal{N}_\theta^H])(\widehat{X}_s) \le u_s^\top\hat{y}(\widehat{X}_s) + c_0 + \varepsilon$. 
Since $t \wedge \tau_{\mathcal{K}}^{\widehat{X}} \le t$, we obtain item (2).
\end{pf}

\begin{rem}
Corollary~\ref{cor:Weak_passivity} is a localized statement: 
it applies to the stopped process $\widehat{X}_{t\wedge\tau_{\mathcal{K}}^{\widehat{X}}}$ and includes the additive term $(c_0+\varepsilon)t$, which accumulates linearly in time. 
It provides an a priori energy/passivity control on a prescribed finite horizon and within the region $\mathcal{K}$ where the model is required to operate. 
\end{rem}

\begin{cor}\label{cor:Weak_passivity_strict}
Assume the hypotheses of Theorem~\ref{thm:Structured_UAT_SPHS} and use the notation
of Corollary~\ref{cor:Weak_passivity}. Assume that there exists $\delta > 0$ such that for all $x \in \mathcal{K}$,
\begin{equation*}
r(x) = \frac{1}{2} \Tr(\sigma\sigma^\top\nabla^2H) (x) - \nabla H(x)^\top R(x) \nabla H(x) \leq - \delta.
\end{equation*}
Choose the networks from Theorem~\ref{thm:Structured_UAT_SPHS} with tolerance $\eta > 0$ small enough so that $\sup_{x \in \mathcal{K}}|\hat{r(}x) - r(x)| \le \frac{\delta}{2}$ (for example $C_2 \eta \leq \frac{\delta}{2}$
with $C_2$ as in the proof of Corollary~\ref{cor:Weak_passivity}). 
Then $\hat{r}(x) \le0$ on $\mathcal{K}$ and, for all $t \in [0,T]$,
\begin{equation*}
\mathbb{E} [\mathcal{N}_\theta^H( \widehat{X}_{t \wedge \tau_{\mathcal{K}}^{\widehat{X}}})]
\le \mathcal{N}_\theta^H (x_0)
+ \mathbb{E} \int_0^{t \wedge \tau_{\mathcal{K}}^{\widehat{X}}}u_s^\top \hat{y}(\widehat{X}_s) \, ds.
\end{equation*}
In particular, the learned SPHS is weakly passive on $\mathcal{K}$ (up to the exit time) in the sense of Definition~\ref{def:strong-weak-passivity}.
\end{cor}

\begin{pf}
By Theorem~\ref{thm:Structured_UAT_SPHS} and the estimate in the proof of
Corollary~\ref{cor:Weak_passivity}, choose $\eta > 0$ such that
$\sup_{x \in \mathcal{K}}|\hat{r}(x) - r(x)| \le \frac{\delta}{2}$.
Since $r(x) \le - \delta$ on $\mathcal{K}$, we obtain $\hat{r}(x) \le - \frac{\delta}{2} \le0$
for all $x\in\mathcal{K}$.
Apply It\^{o}'s formula to $\mathcal{N}_\theta^H(\widehat{X}_{t \wedge \tau_{\mathcal{K}}^{\widehat{X}}})$.
As in Corollary~\ref{cor:Weak_passivity}, the local martingale term has zero expectation, hence
\begin{equation*}
\begin{split}
\mathbb{E}\big[\mathcal{N}_\theta^H(\widehat{X}_{t \wedge \tau_{\mathcal{K}}^{\widehat{X}}}) \big]
& = \mathcal{N}_\theta^H(x_0)
+ \mathbb{E} \int_0^{t \wedge \tau_{\mathcal{K}}^{\widehat{X}}}u_s^\top \hat{y}(\widehat{X}_s) \, ds \\
&\qquad + \mathbb{E} \int_0^{t\wedge \tau_{\mathcal{K}}^{\widehat{X}}} \hat{r}( \widehat{X}_s) \,ds.
\end{split}
\end{equation*}
For $s \le\tau_{\mathcal{K}}^{\widehat{X}}$ we have $\widehat{X}_s \in \mathcal{K}$, hence
$\hat{r}(\widehat{X}_s) \le0$ and therefore the last term is $\le0$. 
Dropping it proves the claim.
\end{pf}

\begin{prop}
\label{prop:Weak_passivity_global}
Assume the hypotheses of Theorem~\ref{thm:Structured_UAT_SPHS} and let $\widehat{X}$ denote the corresponding learned SPHS solution on $[0,T]$ with deterministic initial condition $x_0$.
Let
\begin{equation*}
\begin{cases}
\hat{y}(x) &:= g(x)^\top \nabla \mathcal{N}_\theta^H(x), \\[6pt]
\hat{r}(x) &:= \frac{1}{2} \Tr\Big(\mathcal{N}_\theta^\sigma(x)\big(\mathcal{N}_\theta^\sigma(x) \big)^\top \nabla^2 \mathcal{N}_\theta^H(x)\Big)\\[6pt]
&\qquad - \nabla \mathcal{N}_\theta^H(x)^\top \mathcal{N}_\theta^R (x) \nabla \mathcal{N}_\theta^H(x)
\end{cases}
\end{equation*}
Assume that $\mathcal{N}_\theta^H \in C^2(\mathbb{R}^n)$ and that the following two conditions hold:
\begin{enumerate}
\item[(i)] $\hat{r}(x)\le0$ for all $x\in\mathbb{R}^n$.
\item[(ii)] For every $t\in[0,T]$,

\begin{equation*}
\begin{cases}
&\mathbb{E}\int_0^t \Big\|\big(\nabla\mathcal{N}_\theta^H(\widehat{X}_s)\big)^\top
\mathcal{N}_\theta^\sigma(\widehat{X}_s)\Big\|^2\,ds<\infty,\\[6pt]
&\mathbb{E}\int_0^t \big|u_s^\top \hat{y}(\widehat{X}_s)\big|\,ds < \infty,
\end{cases}
\end{equation*}
and $\mathbb{E} \big[|\mathcal{N}_\theta^H(\widehat{X}_t)|\big] < \infty$.
\end{enumerate}
Then, for all $t \in [0,T]$,
\begin{equation*}
\mathbb{E}\big[ \mathcal{N}_\theta^H(\widehat{X}_t) \big]
\leq \mathcal{N}_\theta^H(x_0) + \mathbb{E}\int_0^{t} u_s^\top \hat{y}(\widehat{X}_s) \, ds.
\end{equation*}
In particular, the learned SPHS is weakly passive without localization (on $[0,T]$) in the sense of Definition~\ref{def:strong-weak-passivity} (with storage $\mathcal{N}_\theta^H$ and output $\hat{y}$).
If the assumptions hold for every finite horizon $T > 0$, then the same inequality holds for all $t \ge0$.
\end{prop}

\begin{pf}
By It\^o's formula and the SPHS structure,
\begin{equation*}
   \mathcal{N}_\theta^H(\widehat{X}_t) = \mathcal{N}_\theta^H(x_0) 
+ \int_0^t u_s^\top \hat{y}(\widehat{X}_s)\, ds 
+ \int_0^t\hat{r}(\widehat{X}_s)\, ds + M_t,
\end{equation*}
where
\begin{equation*}
   M_t = \int_0^t \big(\nabla \mathcal{N}_\theta^H(\widehat{X}_s)\big)^\top
       \mathcal{N}_\theta^\sigma (\widehat{X}_s)\, dW_s.
\end{equation*}
Assumption (ii) implies that $M_t$ is a square-integrable martingale, hence $\mathbb{E}[M_t] = 0$ and $\mathbb{E}[|M_t|] \le (\mathbb{E}[M_t^2])^{\frac{1}{2}} < \infty$. 
Moreover, by (ii),
$\mathbb{E} \int_0^t|u_s^\top \hat{y}(\widehat{X}_s)| \,ds < \infty$
and $\mathbb{E} [|\mathcal{N}_\theta^H(\widehat{X}_t)|] < \infty$, so

\begin{align*}
\mathbb{E}\Big[\Big|\int_0^t\hat{r}(\widehat{X}_s)\,ds\Big|\Big]
&\le \mathbb{E}[|\mathcal{N}_\theta^H(\widehat{X}_t)|]+|\mathcal{N}_\theta^H(x_0)|\\
&+\mathbb{E}\int_0^t|u_s^\top\hat{y}(\widehat{X}_s)|\,ds + \mathbb{E}[|M_t|]<\infty.
\end{align*}

Taking expectations in the It\^o identity yields
\begin{equation*}
   \mathbb{E}[\mathcal{N}_\theta^H(\widehat{X}_t)] 
   = \mathcal{N}_\theta^H(x_0) + \mathbb{E} \int_0^t u_s^\top \hat{y}(\widehat{X}_s)\, ds + \mathbb{E} \int_0^t\hat{r}(\widehat{X}_s)\,ds.
\end{equation*}
Using (i) we have $\hat{r}(\widehat{X}_s) \le0$ a.s.\ for all $s$, hence $\int_0^t \hat{r}(\widehat{X}_s)\,ds \le0$ a.s.\ and therefore $\mathbb{E} \int_0^t\hat{r}(\widehat{X}_s)\, ds \le0$. 
Discarding this non-positive term yields the claim.
\end{pf}


\section{Discussion}\label{sec:discussion}
Numerical evidence indicates that incorporating the complete port-Hamiltonian structure into the learning architecture provides clear practical benefits. 
Over extended periods, energy remains within one order of magnitude of the actual value, phase-space trajectories remain on the correct invariant manifold, and the model is less susceptible to noise than an unconstrained multilayer perceptron baseline.

However, several challenges remain.  
Training costs increase disproportionately with state dimension because each iteration requires Hessian-vector products. Scalable parameterizations or sparsity priors are necessary for high-dimensional systems. 
Additionally, the present work assumes full-state observations. 
Extending the formulation to partial and irregular data will likely necessitate a latent SDE approach. 
Finally, although the model is less fragile than unconstrained baselines, careful tuning of learning rates and loss weights is still necessary. This suggests that automatic hyperparameter scheduling is a promising area for future research.

\section{Conclusion}\label{sec:conclusion}
This paper introduces a unified framework for stochastic port-Hamiltonian neural networks, combining data-driven learning with a structure-preserving parameterization of stochastic port-Hamiltonian systems. 
The Hamiltonian, interconnection, dissipation, and diffusion components can, in principle, be learned from data. 
In our experiments, we fix the analytical interconnection $J$ (and $g$, if applicable) and focus on learning $H_\theta$, though the theory covers the general case in which all coefficients are learned.
A universal approximation theorem demonstrates the expressive power of the architecture. Experiments on noisy mass-spring, Duffing, and Van der Pol oscillators demonstrate an order-of-magnitude improvement in long-term accuracy compared to a baseline multilayer perceptron.

These results suggest that combining geometric structure with neural function approximation is a viable approach for creating reliable, data-driven models of complex physical systems subject to uncertainty. 
These results also motivate future work on scalable variants, partial-observation settings, and integration with feedback control.

\begin{ack}   
M.~Ehrhardt acknowledges funding by the Deutsche Forschungsgemeinschaft (DFG, German Research Foundation) – Project-ID 531152215 – CRC 1701.
\end{ack}

\bibliographystyle{plain}        
\bibliography{references}           


                  
\end{document}